\documentclass[journal]{IEEEtran}

\newif\ifMakeReviewDraft
\MakeReviewDraftfalse

\newif\ifUseColorLinks
\UseColorLinkstrue

\newif\ifDisplayMyComments

\DisplayMyCommentstrue
\usepackage{latexsym}
\usepackage{natbib}

\usepackage{graphicx}

\graphicspath{{./figures/}}

\usepackage{epstopdf}

\usepackage{soul}

\usepackage{algorithm}
\usepackage{algorithmic}

\usepackage{amsmath}

\usepackage{amssymb}

\usepackage{amsthm}
\newtheorem{thm}{Theorem}

\newtheorem{lem}[thm]{Lemma}

\ifMakeReviewDraft

\usepackage{showkeys}

\fi

\usepackage{booktabs}

\usepackage{subfigure} 

\usepackage{placeins}

\usepackage{verbatim}					

\usepackage{acronym}

\usepackage{color}

\usepackage{url}

\usepackage{hyperref}

\hypersetup{
	pdftitle={},    
	pdfsubject={},	
	pdfauthor={},		
	pdfkeywords={}	
	pdfcreator={LaTeX 2e},	
	pdfproducer={},	
	plainpages=false,			%
	unicode=false,				
	pdftoolbar=true,			
	pdfmenubar=true,			
	pdffitwindow=false,		
	pdfstartview={FitH},	
	pdfnewwindow=true,		
	colorlinks=true				
}
\ifUseColorLinks

\hypersetup{
	linkcolor=blue,			
	citecolor=green,		
	filecolor=magenta,	
	urlcolor=cyan				
}

\else

\hypersetup{
	linkcolor=black,		
	citecolor=black,		
	filecolor=black,		
	urlcolor=black			
}

\fi

\definecolor{gray}{rgb}{0.5,0.5,0.5}	
\definecolor{red}{rgb}{1.0,0.0,0.0}	  
\definecolor{green}{rgb}{0.0,1.0,0.0}	  
\definecolor{blue}{rgb}{0.0,0.0,1.0}	  

\newcommand{\ie}{\textit{i.e.}}
\newcommand{\eg}{\textit{e.g.}}
\newcommand{\wrt}{w.{}r.{}t.}

\newacro{SVM}{Support Vector Machine}
\newacro{KKT}{Karush-Kuhn-Tucker}
\newacro{KNN}{$k$-nearest neighbor}
\newacro{PSD}{Proximal Subgradient Descent}
\newacro{DML}{Distance Metric Learning}
\newacro{R2LML}[R\textsuperscript{2}LML]{Reduced-Rank Local Metric Learning}
\newacro{T-R2LML}[T-R\textsuperscript{2}LML]{Transductive Reduced-Rank Local Metric Learning}
\newacro{E-R2LML}[E-R\textsuperscript{2}LML]{Efficient Reduced-Rank Local Metric Learning}
\newacro{EDM}{Euclidean Distance Matrix}
\newacro{MM}{Majorization Minimization}
\newacro{NSF}{National Science Foundation}
\newacro{BCD}{Block Coordinate Descent}

\newcommand{\eref}[1]{Eq.~(\ref{#1})}
\newcommand{\fref}[1]{Figure~\ref{#1}}

\newcommand{\tref}[1]{Table~\ref{#1}}
\newcommand{\sref}[1]{Section~\ref{#1}}
\newcommand{\aref}[1]{Algorithm~\ref{#1}}
\newcommand{\pref}[1]{Problem~(\ref{#1})}

\newcommand{\thmref}[1]{Theorem~\ref{#1}}
\newcommand{\lemmaref}[1]{Lemma~\ref{#1}}

\ifDisplayMyComments

\newcommand{\mycomment}[1] { \noindent {\small \textcolor{blue}{\textbf{Me:} \emph{#1} } } }
\newcommand{\gcacomment}[1] { \noindent {\small \textcolor{red}{\textbf{GCA:} \emph{#1} } } }
\newcommand{\congcomment}[1] { \noindent {\small \textcolor{cyan}{\textbf{Cong:} \emph{#1} } } }
\else

\newcommand{\mycomment}[1] {}
\newcommand{\gcacomment}[1] {}
\newcommand{\congcomment}[1] {}

\fi

\begin{document}

\title{Reduced-Rank Local Distance Metric Learning \\ for k-NN Classification}


\author{Yinjie Huang \and Cong Li \and Michael Georgiopoulos \and Georgios C. Anagnostopoulos}


\thanks{Y. Huang, C. Li, M. Georgiopoulos was with Department of Electrical Engineering \& Computer Science, University of Central Florida, 4000 Central Florida Blvd, Orlando, Florida, 32816, USA}
\thanks{G. C. Anagnostopoulos was with Department of Electrical and Computer Engineering, Florida Institute of Technology, 150 W University Blvd, Melbourne, Florida, 32901, USA}

\date{Received: date / Accepted: date}

\maketitle

\begin{abstract}
	We propose a new method for local distance metric learning based on sample similarity as side information. These local metrics, which utilize conical combinations of metric weight matrices, are learned from the pooled spatial characteristics of the data, as well as the similarity profiles between the pairs of samples, whose distances are measured. The main objective of our framework is to yield metrics, such that the resulting distances between similar samples are small and distances between dissimilar samples are above a certain threshold. For learning and inference purposes, we describe a transductive, as well as an inductive algorithm; the former approach naturally befits our framework, while the latter one is provided in the interest of faster learning. Experimental results on a collection of classification problems imply that the new methods may exhibit notable performance advantages over alternative metric learning approaches that have recently appeared in the literature\footnote{A preliminary version of the work presented here has appeared in \cite{Huang2013}.}.
	
\end{abstract}


\section{Introduction}
\label{sec:Introduction}

Distance computations underlie many machine learning approaches with the \ac{KNN} decision rule for classification and the $k$-Means algorithm for clustering problems being the two most prominent examples. 
Such computations are often, if not mainly, performed using the ordinary Euclidean metric or a weighted variation of it, namely the Mahalanobis distance.  
However, employing fixed, global metrics, such as the ones just mentioned, for computing distances may not yield good results in all settings. This fact motivated many researchers to pursue data-driven approaches, in order to infer the best metric for a given problem (\eg\ \cite{Xing2002} and \cite{ShalevShwartz2004}). In successfully addressing this task, one needs to take into account the data's distributional characteristics and to take advantage of any \emph{side information} that may be available for the data. In general, such approaches are referred to as \emph{metric learning}. A typical instance of such an approach is to learn the weight matrix of the Mahalanobis metric, which occasionally we will refer to it simply as the metric. Equivalently, this task could be viewed as follows: a de-correlating linear transformation of the data is learned in the native space and Euclidean distances are computed in the range space of the learned linear transform (feature space). When dealing with a classification problem, a \ac{KNN} algorithm based on the learned metric is eventually employed to label samples.

Our work falls under the metric learning approaches for classification tasks, where the Mahalanobis metric is learned through the help of pair-wise sample similarities. By assumption, two samples will be similar, if they feature the same class label. The goal of similarity-based metric learning is to map similar samples close and to map dissimilar samples far apart in the feature space. After learning this metric, an eventual application of a \ac{KNN} decision rule exhibits improved performance over a direct application of the same rule using the Euclidean metric.

Many metric learning algorithms have been proposed and show significant improvements over the Euclidean \ac{KNN} rule. For example, in \cite{Xing2002}, the authors posed similarity-based metric learning as a convex optimization problem, which is employed in a clustering problem. A projected gradient ascent algorithm is utilized to optimize the problem. \cite{ShalevShwartz2004} described an online algorithm for supervised learning of metrics. Their algorithm is based on successive projections onto the positive semi-definite cone. They also offered a dual version of the algorithm which is able to incorporate kernel operators. Moreover, Neighborhood Components Analysis (NCA) \cite{Goldberger2004}, maximizes the leave-one-out performance on the training data based on stochastic nearest neighbors. Their classification model is non-parametric, making no assumptions about the shape of the class distributions. \cite{Chopra2005} built a system that maps images to points in a lower dimensional space so that these points lie closer, if the original images are similar. This model consists of two convolutional neural networks to address geometric distortions. Furthermore, Large Margin Nearest Neighbor (LMNN) \cite{Weinberger2006} is trying to learn the metric so that the $k$-nearest neighbors of each sample belong to the same class, while others are separated by a large margin. They cast their optimization as an instance of semi-definite programming. Finally, \cite{Davis2007} formulated the problem using information entropy and introduce Information Theoretic Metric Learning (ITML). ITML tries to minimize the differential relative entropy between two multivariate Gaussian distributions with distance metric constraints.  

\begin{figure*}[ht]
	\vskip 0.2in
	\begin{center}
		\centerline{\includegraphics[width=\textwidth]{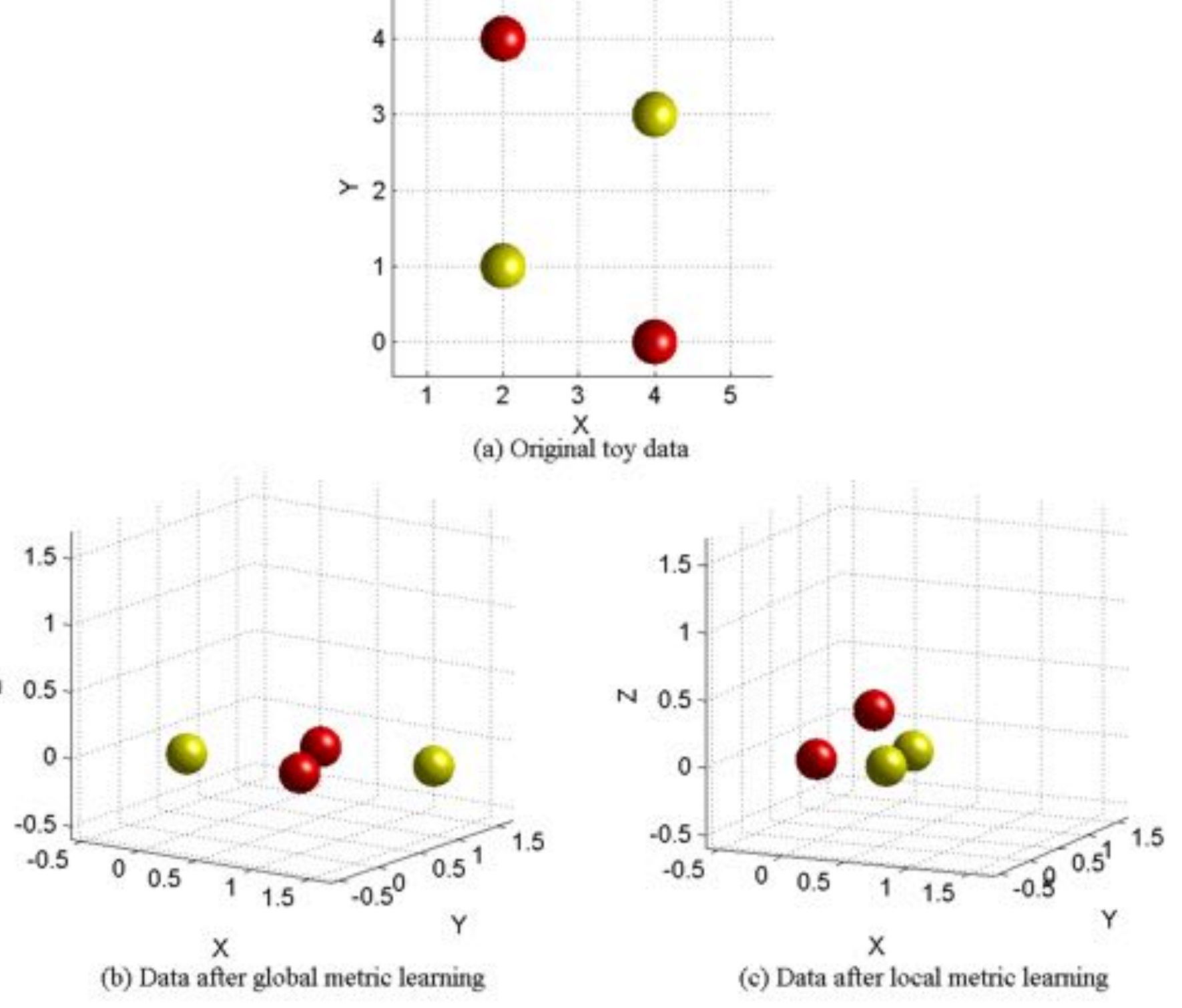}}
		\caption{Toy dataset that illustrates the potential advantages of learning a local metric instead of a global one. (a) Original data distribution. (b) Data distribution in the feature space obtained by learning a global metric. (c) Data distribution in the feature space obtained by learning similarity-based local metrics.}
		\label{figure1}
	\end{center}
	\vskip -0.2in
\end{figure*} 

The previous metric learning approaches share one common feature: they employ a single, global metric, \ie, a metric that is used for all distance computations. However, this global metric learning approach may not be well-suited to some multi-modal or non-linear scenarios. \fref{figure1} illustrates this point via a toy dataset containing $4$ samples from two classes. Note that this toy problem is merely a conceptual device that shows the comparison of what a global metric and local metrics will do. \fref{figure1}(a) shows the samples in their native space. \fref{figure1}(b) shows the feature space resulting from learning a global metric, while \fref{figure1}(c) shows the transformed data after learning two local metrics, which take into account the location and similarity characteristics of the data involved. We refer to such metrics as \emph{local metrics}. In contrast to the result obtained using a global metric, local metrics can map similar samples closer to each other, as shown in \fref{figure1}(c). This may potentially improve $1$-NN classification performance, when compared to the sample distributions in the other two cases.  

Many local metric learning algorithms have been proposed. In \cite{Hastie1996}, local metrics are determined from centroid information. The neighborhoods are shrank in directions that are orthogonal to the local decision boundaries, while elongated in directions parallel to the boundaries. In \cite{Bilenko2004}, the authors introduced a clustering framework, in which a local metric is defined for each cluster. \cite{Yang2006} proposed a local metric learning model that generates distance metrics to accommodate multiple modes for each class. Moreover, an Expectation-Maximization-like algorithm is employed to solve their probabilistic framework. In \cite{Weinberger2008}, the authors of LMNN  developed the LMNN-Multiple Metric (LMNN-MM) approach. When applied in a classification context, the number of metrics equals the number of classes. Additionally, \cite{Noh2010} proposed Generative Local Metric Learning (GLML), which learns local metrics through NN classification error minimization. GLML assumes that the data has been drawn from a Gaussian mixture, which is a rather strong assumption. Eventually, \cite{Wang2012} proposed Parametric Local Metric Learning (PLML), in which each local metric is defined in relation to an anchor point of the instance space. In order to solve their local metric problem, they employ a projected gradient method to optimize their large-margin objective. \cite{Zhu2014}'s model learns multiple distance metrics under different scales of the data and combine the decisions from these learned metrics. Finally, they formulated the local metric learning problem as a \ac{SVM} model.

In this paper, we propose a new local metric learning approach, which we will refer to as \ac{R2LML}. As elaborated in \sref{sec:ProblemFormulation}, in our approach, the local Mahalanobis metric (in specific, its weight matrix) is modeled as a conical combination of positive semi-definite weight matrices. With the assistance of pair-wise similarities, both the weight matrices and their coefficients are learned from the data. The weight matrices themselves correspond to local linear transformations of the original data from their native space into a locality-dependent feature space. These transformations are learned such that similar (dissimilar) samples map close to (far from) each other, so that they exhibit small (large) pair-wise Euclidean distances in these locally-defined feature spaces. Note that, in our case, we will consider samples to be similar, if they share the same label. Moreover, we will consider two variants of \ac{R2LML}. The first one, namely \ac{T-R2LML}, uses transductive learning \cite{Vapnik1998} to infer the test sample coefficients necessary for defining the local metrics. The second one, which is referred to as \ac{E-R2LML}, aims to address the computationally intensive nature of the first variant. As discussed in \sref{sec:ProblemFormulation}, it employs a technique first used in \cite{Wang2012}, according to which the coefficients of a test sample are set equal to the ones of its nearest (in terms of Euclidean distance) training sample. Finally, it is worth mentioning that both variants employ a sum-of-nuclear-norms regularizer to avoid over-fitting, when warranted. 


In order to optimize the aforementioned formulations, two efficient \ac{BCD} algorithms are presented in \sref{sec:Algorithm}. In specific, as delineated in \sref{sec:TwoStepAlgorithm}, a two-block minimization algorithm is able to solve the \ac{E-R2LML} learning problem. The first block minimization with respect to the weight matrices constitutes a \ac{PSD} step, which is able to cope with the non-smooth nature of the formulation's regularizer. The second block minimization, which attempts to optimize the metric coefficients, constitutes a straightforward \ac{MM} step. On the other hand, the algorithm intended for solving the \ac{T-R2LML} formulation differs from the first one in that it includes an additional block minimization with respect to the test samples' similarities. As shown in \sref{sec:ThreeStepAlgorithm}, the relevant optimization, while addressing a binary integer programming problem, can be efficiently performed. The convergence analysis for both methods is showcased in \sref{sec:Analysis}.    


Finally, in \sref{sec:Experiments}, the first experiment studies the importance of regularization in the proposed frameworks based on the synthetic datasets. Additionally, the relationship between the number of local metrics and the accuracies is highlighted in the second experiment. Eventually, we demonstrate the capabilities of \ac{T-R2LML} and \ac{E-R2LML} with respect to classification tasks. When compared to other recent global or local metric learning approaches, \ac{T-R2LML} and \ac{E-R2LML} achieve the highest classification accuracy in $9$ and $14$ out of $18$ datasets respectively. 


\section{Problem Formulation}
\label{sec:ProblemFormulation}

Define $\mathbb{N}_M \triangleq \{ 1, 2, \ldots , M\}$ for any positive integer $M$. Suppose we have $n$ input training set $\{ \boldsymbol{x}_n \in \mathbb{R}^{D} \}_{n \in \mathbb{N}_N}$ and an accompanying similarity matrix $\boldsymbol{S} \in \left\{0, 1\right\}^{N \times N}$ as side information, in which each entry represents a corresponding pair-wise sample similarity. If $\boldsymbol{x}_m$ and $\boldsymbol{x}_n$ are similar, then $s_{mn} = 1$; otherwise, then $s_{mn} = 0$. In a classification context, two samples from the same (or different) class can be naturally deemed similar (or dissimilar).

The Mahalanobis distance between two samples $\boldsymbol{x}_n$ and $\boldsymbol{x}_m$ is $d_{\boldsymbol{A}} (\boldsymbol{x}_m,\boldsymbol{x}_n) $ $\triangleq \sqrt{(\boldsymbol{x}_m-\boldsymbol{x}_n)^T\boldsymbol{A}(\boldsymbol{x}_m-\boldsymbol{x}_n)}$. We will refer to $\boldsymbol{A} \in \mathbb{R}^{D \times D}$ (a positive semi-definite matrix, denoted as $\boldsymbol{A} \succeq 0$) as the \emph{weight matrix of the metric}. When $\boldsymbol{A}=\boldsymbol{I}$, the previous metric, obviously, becomes the Euclidean distance metric. Since any positive semi-definite weight matrix can be expressed as $\boldsymbol{A}=\boldsymbol{L}^T\boldsymbol{L}$, where $\boldsymbol{L} \in \mathbb{R}^{P \times D}$ with $P \leq D$, the previously defined Mahalanobis distance can be expressed as $d_{\boldsymbol{A}}(\boldsymbol{x}_m,\boldsymbol{x}_n)=\left \| \boldsymbol{L}(\boldsymbol{x}_m-\boldsymbol{x}_n)\right \|_2$. This last expression implies that the Mahalanobis distance based on $\boldsymbol{A}$ between two points in the native space can be viewed as the Euclidean distance between the corresponding points in the feature space obtained through the linear transformation $\boldsymbol{L}$. 

Metric learning approaches are trying to learn $\boldsymbol{A}$ so to minimize the distances between pairs of similar points, while maximizing, or maintaining above a certain threshold, the distances between dissimilar points in the feature space. The problem can be formulated as follows:

\begin{align}
\label{eq:eqn01}
\underset{\boldsymbol{A} \succeq 0}{min} \ &  \ \underset{m,n}{\sum} s_{mn} d_{\boldsymbol{A}}(\boldsymbol{x}_m,\boldsymbol{x}_n) \\
s.t.             \ & \ \underset{m,n}{\sum} (1-s_{mn})d_{\boldsymbol{A}}(\boldsymbol{x}_m,\boldsymbol{x}_n) \geq 1. \nonumber
\end{align}

\noindent
\pref{eq:eqn01} is a semi-definite programming problem involving a global metric based on $\boldsymbol{A}$. Several approaches like LMNN, ITML and NCA are learning a single global metric. However, as argued earlier via \fref{figure1}, a global metric may not be advantageous under all circumstances. 

In this paper, we propose \ac{R2LML}, a new local metric approach. We assume that the metric involved is expressed as a conical combination of $K \geq 1$ Mahalanobis metrics. The metric between $\boldsymbol{x}_n$ and $\boldsymbol{x}_m$ is defined as $\sum_k \boldsymbol{A}^k g^k_m g^k_n$. Here, $\boldsymbol{g}^k \in \mathbb{R}^{N}$ is a vector for each local metric $k$, of which the $n^{th}$ element $g^k_n$ may be considered as a measure of how pertinent the $k$\textsuperscript{th} metric is, when computing distances involving the $n^{th}$ sample. 
Not only do these metrics change throughout the input space along the data\rq s underlying manifold, but are also affected by the similarity of nearby samples. Note that these coefficient vectors will be also unknown for test samples and, hence, need to be inferred as well. A natural avenue to achieve this is via a transductive learning scheme.  


The metric $\sum_k \boldsymbol{A}^k g^k_m g^k_n$ is actually a semi-metric \cite{Sefer2011}, which violates the triangle inequality. When choosing $\boldsymbol{g}$ properly, there exists triplets of samples that does not satisfy the triangle inequality in the feature space. However, in our experiments, it seems that a proper metric is almost always learned. For example, when considering the \textit{Pendigits} dataset (containing about $200$ samples), the triangle inequalities that we examined (over one million) were all satisfied. In the rest of our work, we still refer this semi-metric as metric for simplicity.

Transductive learning trains both labeled and unlabeled data to yield improved performance. According to \cite{Vapnik1998}, when solving a problem, one should avoid inferring a function as an intermediate step. There are many transductive learning approaches proposed for various algorithms. In \cite{Bennett1999}, \cite{Chen2002}, \cite{Gammerman2013} and \cite{Joachims1999}, the authors developed transductive learning framework for Support Vector Machine. \cite{Joachims2003} and \cite{Kukar2002} designed transductive algorithm for \ac{KNN} classifiers and general classifiers respectively. There are also transductive learning approaches for graph-based models in \cite{Talukdar2009}, \cite{Liu2009} and \cite{Zhou2007}.

In \ac{T-R2LML}, the input training set $\{ \boldsymbol{x}_n \in \mathbb{R}^{D} \}_{n \in \mathbb{N}_N}$ and test set $\{ \boldsymbol{x}_n \in \mathbb{R}^{D} \}_{n \in \mathbb{N}_M}$ are combined. Since labels of test samples are unknown, the entries of the similarity matrix $\boldsymbol{S} \in \left\{0, 1\right\}^{(N+M) \times (N+M)}$ that involve test data are randomly initialized. The vectors $\boldsymbol{g}^k$ belong to $\Omega^{'}_{g} \triangleq \left\{ \left\{ \boldsymbol{g}_k \right\}_{k \in \mathbb{N}_K} \in \left[0, 1\right]^{N+M} : \boldsymbol{g}^k \succeq \boldsymbol{0}, \ \sum_k \boldsymbol{g}^k = \boldsymbol{1} \right\}$, where '$\succeq$' denotes component-wise ordering. The $\boldsymbol{g}^k$s' need to sum up to the all-ones vector $\boldsymbol{1}$, so that at least one metric is relevant, when computing distances from each sample. Obviously, if $K=1$, $\boldsymbol{g}^1 = \boldsymbol{1}$, which amounts to learning a single global metric.

Based on the previous description, the weight matrix for each pair $(m,n)$ is defined as $\sum_k \boldsymbol{A}^k g^k_m g^k_n$. Note that the distance between every pair of points features a different weight matrix. We now consider the following formulation motivated by \pref{eq:eqn01}, which varies over $k \in \mathbb{K}$:

\begin{align}
\label{eq:eqn03}
\underset{\boldsymbol{L}^k \boldsymbol{S}, \boldsymbol{g}^k \in \Omega^{'}_{g}, \xi^k_{m,n} \geq 0}{min} &  \  \underset{k}{\sum} \underset{m,n}{\sum} s_{mn}\left \| \boldsymbol{L}^k \Delta \boldsymbol{x}_{mn} \right \|^2_2 g_n^k g_m^k + \\
\ & \ +C \underset{k}{\sum} \underset{m,n}{\sum} (1-s_{mn}) \xi^k_{mn} + \lambda \underset{k}{\sum} \mbox{rank} (\boldsymbol{L}^k) \nonumber \\
s.t. \ & \ \left \| \boldsymbol{L}^k \Delta \boldsymbol{x}_{mn} \right \|^2_2 \geq 1-\xi^k_{mn}, \nonumber \\
\ & \ m, n \in \mathbb{N}_{N+M}, \ k \in \mathbb{N}_K \nonumber \\
\ & \ s_{mn} \in \{0,1\}, \ m,n \in  \mathbb{N}_M \nonumber \\
\ & \ s_{mm}=1, \ s_{mn} = s_{nm}, \ m,n \in  \mathbb{N}_M  \nonumber \\
\ & \ \underset{n \in \mathbb{N}_{N+M}}{\sum s_{mn}}  \geq 2, \ m \in \mathbb{N}_M, \nonumber
\end{align}

\noindent
where $\Delta \boldsymbol{x}_{mn} \triangleq \boldsymbol{x}_m-\boldsymbol{x}_n$ and $\mbox{rank} (\boldsymbol{L}^k)$ denotes the rank of matrix $\boldsymbol{L}^k$. In the objective function, the first term attempts to minimize the distance between similar samples, while the second term along with the first set of soft constraints (due to the slack variables $\xi^k_{mn}$) encourage distances between pairs of dissimilar samples to be larger than $1$. 
Evidently, $C > 0$ controls the penalty of violating the previous prerequisite. Finally, the last term penalizes large ranks of the linear transformations $\boldsymbol{L}^k$. Therefore, the regularization parameter $\lambda \geq 0$ essentially controls the dimensionality of the feature space. As is typical for identifying good values for regularization parameters, both $C$ and $\lambda$ are chosen via a validation procedure. Note that the diagonal elements are all set to $1$ in the similarity matrix. Finally, the last constraint guarantees that the testing samples include all the labels of the training set. 

Via the use of the hinge function, $[u]_+ \triangleq \max \{u,0\}$ for all $u \in \mathbb{R}$, \pref{eq:eqn03} can be reformulated by eliminating the slack variables. Notice that $\mbox{rank} (\boldsymbol{L}^k)$ is a non-convex function \wrt\ $\boldsymbol{L}^k$ and, hence, is hard to optimize. Following the approaches of \cite{Cand`es2009} and \cite{Cand`es2008}, $\mbox{rank} (\boldsymbol{L}^k)$ can be replaced with its convex envelope, \ie, $\boldsymbol{L}^k$'s nuclear norm. The new problem is now formulated as:

\begin{align}
\label{eq:eqn04}
\underset{\boldsymbol{L}^k, \boldsymbol{S}, \boldsymbol{g}^k \in \Omega^{'}_{g}}{min} \ & \ \underset{k} {\sum} \underset{m,n}{\sum}  s_{mn}\left \| \boldsymbol{L}^k \Delta \boldsymbol{x}_{mn}\right  \|^2_2 g_n^k g_m^k + \\
\ & \ + C(1-s_{mn}) \left[ 1-\left \| \boldsymbol{L}^k \Delta \boldsymbol{x}_{mn} \right \|^2_2\right]_+ \nonumber \\
& + \lambda \underset{k}{\sum} \left \| \boldsymbol{L}^k \right \|_* \nonumber \\
s.t. \ & \ s_{mn} \in \{0,1\}, \ m,n \in  \mathbb{N}_M \nonumber \\
\ & \ s_{mm}=1, \ s_{mn} = s_{nm}, \ m,n \in  \mathbb{N}_M \nonumber \\
\ & \ \underset{n \in \mathbb{N}_{N+M}}{\sum s_{mn}}  \geq 2, \ m \in \mathbb{N}_M, \nonumber
\end{align}

\noindent
where $\left \| \cdot \right \|_*$ denotes the nuclear norm, in specific, $\left \| \boldsymbol{L}^k \right \|_* \triangleq \sum_{s=1}^P \sigma_s(\boldsymbol{L}^k)$, where $\sigma_s$ is a singular value of $\boldsymbol{L}^k$. 

A shortcoming of \ac{T-R2LML} is that, it is computationally intensive, since the computation of the gradient in each step requires $O(K(M + N)^2)$ operations and, typically, $M >> N$. Hence, we are also inclined to consider a faster, albeit approximate, approach to address our local metric learning problem. 
In specific, as done in \cite{Wang2012}, for each test sample $\boldsymbol{x}$, its $\boldsymbol{g}$ vector will be assigned the value of the corresponding vector associated to $\boldsymbol{x}$'s nearest (in terms of Euclidean distance) training sample. We refer to this model as \ac{E-R2LML} and its training only requires $O(KN^2)$ operations per step.

For \ac{E-R2LML}, $\boldsymbol{g}^k$ belongs to $\Omega_{g} \triangleq \left\{ \left\{ \boldsymbol{g}_k \right\}_{k \in \mathbb{N}_K} \in \left[0, 1\right]^N : \boldsymbol{g}^k \succeq \boldsymbol{0}, \ \sum_k \boldsymbol{g}^k = \boldsymbol{1} \right\}$ when considering only the training set. Finally, the problem becomes:

\begin{align}
\label{eq:R2LML01}
\underset{\boldsymbol{L}^k,\boldsymbol{g}^k \in \Omega_{g}}{min} \ & \ \underset{k} {\sum} \underset{m,n}{\sum}  s_{mn}\left \| \boldsymbol{L}^k \Delta \boldsymbol{x}_{mn}\right  \|^2_2 g_n^k g_m^k + \\
\ & \ + C(1-s_{mn}) \left[ 1-\left \| \boldsymbol{L}^k \Delta \boldsymbol{x}_{mn} \right \|^2_2\right]_+ + \lambda \underset{k}{\sum} \left \| \boldsymbol{L}^k \right \|_*. \nonumber 
\end{align}

\section{Algorithm}
\label{sec:Algorithm}

\pref{eq:R2LML01} and \pref{eq:eqn04} reflect minimizations over two and three sets of variables respectively. In \ac{E-R2LML}, for fixed $\boldsymbol{g}^k$, the problem is non-convex \wrt\ $\boldsymbol{L}^k$, since the second term in \eref{eq:R2LML01} is the combination of a convex function (hinge function) and a non-monotone function \wrt\ $\boldsymbol{L}^k$, namely $1-\left \| \boldsymbol{L}^k \Delta \boldsymbol{x}_{mn} \right \|^2_2$. On the other hand, the problem is also non-convex w.r.t $\boldsymbol{g}^k$ for fixed $\boldsymbol{L}^k$, since the similarity matrix $\boldsymbol{S}$ is almost always indefinite, which will be argued in the sequel. Thus, the objective function may have multiple minima and an iterative procedure to minimize it may have to be initialized multiple times with different values for the unknown parameters in order to find a good solution. The same observations apply to \ac{T-R2LML} as well. Finally, notice that, for \ac{T-R2LML}, when optimizing \pref{eq:eqn04} \wrt \ $\boldsymbol{S}$, while holding $\boldsymbol{g}^k$ and $\boldsymbol{L}^k$ fixed, the problem under consideration is convex. In what follows next, we discuss two training algorithms: a two-block \ac{BCD} algorithm for \ac{E-R2LML} and a very similar \ac{BCD} algorithm for \ac{T-R2LML} that can perform the optimizations in question.

\subsection{Two-Block Algorithm for \ac{E-R2LML}}
\label{sec:TwoStepAlgorithm}

We first start off with a discussion of the \ac{BCD} that trains the \ac{E-R2LML} framework. For the first block, we try to solve for every $\boldsymbol{L}^k$ by holding the $\boldsymbol{g}^k$'s fixed. In this case, \pref{eq:R2LML01} becomes an unconstrained minimization problem, which can be expressed in the form $f(\boldsymbol{w})+r(\boldsymbol{w})$, where $\boldsymbol{w}$ is the parameter we are trying to minimize over (in our case, all $\boldsymbol{L}^k$'s). $f(\boldsymbol{w})$ is the non-differentiable hinge loss function, while $r(\boldsymbol{w})$ is a non-smooth, convex regularization term. 
Hence, we resort to using a \ac{PSD} method in a similar fashion as has been done in \cite{Rakotomamonjy2011} and \cite{Chen2009}. It might be worth noting that the particular approach is a special case of the one presented in \cite{Duchi2009}. It is this relationship that we leverage to develop the convergence analysis of our \ac{PSD} steps in \sref{sec:Analysis}.

Next, for the second block we minimize \wrt\ each $\boldsymbol{g}^k$ vector, while the $\boldsymbol{L}^k$'s are assumed to be fixed. Consider a matrix $\boldsymbol{\bar{S}}^k$ associated to the $k^{th}$ metric, whose $(m,n)$ element is defined as: 

\begin{equation}
\label{eq:eqn05}
\bar{s}^k_{mn} \triangleq s_{mn}\left \| \boldsymbol{L}^k \Delta \boldsymbol{x}_{mn} \right \| ^2_2, \ \ \ m, n \in \mathbb{N}_N.
\end{equation}       \\

\noindent
Then, by concatenating all individual $\boldsymbol{g}^k$ vectors into a single vector $\boldsymbol{g} \in \mathbb{R}^{KN}$ and by defining the block-diagonal matrix $\boldsymbol{\tilde{S}}$ as:

\begin{align}
\label{eq:eqn07}
\boldsymbol{\tilde{S}} \triangleq \begin{bmatrix}
\boldsymbol{\bar{S}}^1 	& 0 & ...& 0 \\ 
0         			& \boldsymbol{\bar{S}}^2 	&...      & 0 \\ 
\vdots    							& \vdots    							& \ddots  &\vdots \\ 
0         							& ...       							&0        & \boldsymbol{\bar{S}}^K
\end{bmatrix}
\in \mathbb{R}^{KN \times KN}.
\end{align} 

\noindent
\pref{eq:R2LML01} can be expressed as:

%

\begin{align}
\label{eq:eqn08}
\underset{\boldsymbol{g \in \Omega_g}}{min} \ & \ \boldsymbol{g}^T \boldsymbol{\tilde{S}} \boldsymbol{g},
\end{align}

\noindent
where $\Omega_g = \left\{ \boldsymbol{g} \in \left[ 0, 1 \right]^{KN} : \boldsymbol{g} \succeq \boldsymbol{0}, \ \boldsymbol{B} \boldsymbol{g} = \boldsymbol{1} \right\}$, $ \boldsymbol{B} \triangleq \boldsymbol{1}^T \otimes \boldsymbol{I}_N$ and $\otimes$ denotes the Kronecker product. \pref{eq:eqn08} is non-convex, since $\boldsymbol{\tilde{S}}$ is almost always indefinite. This stems from the fact that $\boldsymbol{\tilde{S}}$ is a block diagonal matrix, whose blocks are Euclidean Distance Matrices (EDMs). EDMs feature exactly one positive eigenvalue (unless all of them equal to $0$). Since each EDM is a hollow matrix, its trace equals to $0$, which implies that its remaining eigenvalues must be negative \cite{Balaji2007}. Therefore, $\boldsymbol{\tilde{S}}$ will feature negative eigenvalues.

In order to minimize \pref{eq:eqn08}, we employ a \ac{MM} approach \cite{Hunter2004}, which requires first identifying a function of $\boldsymbol{g}$ that majorizes the objective function at hand. Let $\mu \triangleq -\lambda_{max}(\boldsymbol{\tilde{S}})$, where $\lambda_{max}(\boldsymbol{\tilde{S}})$ is the largest eigenvalue of $\boldsymbol{\tilde{S}}$. Since $\boldsymbol{\tilde{S}}$ is indefinite, $\lambda_{max}(\boldsymbol{\tilde{S}}) > 0$. Then, $\boldsymbol{H} \triangleq \boldsymbol{\tilde{S}}+\mu \boldsymbol{I}$ is negative semi-definite. Let $q(\boldsymbol{g}) \triangleq \boldsymbol{g}^T\boldsymbol{\tilde{S}}\boldsymbol{g}$ be the cost function of \eref{eq:eqn08}. Note that $(\boldsymbol{g}-\boldsymbol{g}')^T\boldsymbol{H}(\boldsymbol{g}-\boldsymbol{g}') \leq 0$ for any $\boldsymbol{g}$ and $\boldsymbol{g}'$ and we have that $q(\boldsymbol{g}) < -\boldsymbol{g}'^T \boldsymbol{H} \boldsymbol{g}'+2\boldsymbol{g}'^T\boldsymbol{H} \boldsymbol{g}-\mu \left \| \boldsymbol{g} \right \|^2_2$ for all $\boldsymbol{g} \neq \boldsymbol{g}'$ and equality, only if $\boldsymbol{g} = \boldsymbol{g}'$. The right hand side of the aforementioned inequality constitutes $q$'s majorizing function, denoted as $q(\boldsymbol{g}|\boldsymbol{g}')$. The majorizing function is used to iteratively optimize $\boldsymbol{g}$ based on the current estimate $\boldsymbol{g}'$. So we have the following minimization problem, which is convex w.r.t $\boldsymbol{g}$:

\begin{align}
\label{eq:eqn09}
\underset{\boldsymbol{g} \in \Omega_g}{min} \ & \ 2\boldsymbol{g}'^T \boldsymbol{H} \boldsymbol{g} -\mu \left \| \boldsymbol{g} \right \|^2_2.
\end{align}

\noindent
This problem is readily solvable, as the next theorem implies.

\begin{thm}
	\label{thm1}
	Let $\boldsymbol{g}, \boldsymbol{d} \in \mathbb{R}^{KN}$, $\boldsymbol{B} \triangleq \boldsymbol{1}^T \otimes \boldsymbol{I}_N \in \mathbb{R}^{N \times KN}$ and $c>0$. The unique minimizer $\boldsymbol{g}^{*}$ of
	
	\begin{align}
	\label{eq:general}
	\underset{\boldsymbol{g}}{min} \ & \ \frac{c}{2} \left \| \boldsymbol{g} \right \|^2_2 + \boldsymbol{d}^T \boldsymbol{g} \\
	s.t. \ & \  \boldsymbol{B}\boldsymbol{g}=\boldsymbol{1}, \ \boldsymbol{g} \succeq \boldsymbol{0}, \nonumber
	\end{align}
	
	\noindent
	has the form
	
	\begin{equation}
	\label{eq:eqn10}
	g^{*}_i = \frac{1}{c} \left[ (\boldsymbol{B}^T\boldsymbol{\alpha})_i-\boldsymbol{d}_i \right]_+, \ \ \ i \in \mathbb{N}_{KN}, 
	\end{equation}
	
	\noindent
	where $g_i$ is the $i^{th}$ element of $\boldsymbol{g}$ and $\boldsymbol{\alpha} \in \mathbb{R}^N$ is the Lagrange multiplier vector associated to the equality constraint.
\end{thm}

\begin{proof}
	The Lagrangian of \pref{eq:general} is formulated as:
	
	\begin{equation}
	\label{eq:A2}
	L(\boldsymbol{g},\boldsymbol{\alpha},\boldsymbol{\beta}) = \frac{c}{2} \boldsymbol{g}^T \boldsymbol{g} + \boldsymbol{d}^T \boldsymbol{g} + \boldsymbol{\alpha}^T (\boldsymbol{1}-\boldsymbol{B}\boldsymbol{g})-\boldsymbol{\beta}^T \boldsymbol{g},
	\end{equation}
	
	\noindent
	where $\boldsymbol{\alpha} \in \mathbb{R}^{N}$ and $\boldsymbol{\beta} \in \mathbb{R}^{KN}$ with $\boldsymbol{\beta} \succeq \boldsymbol{0}$ are Lagrange multiplier vectors. If we set the partial derivative of $L(\boldsymbol{g},\boldsymbol{\alpha},\boldsymbol{\beta})$ with respect to $\boldsymbol{g}$ to $\boldsymbol{0}$, we readily have
	
	\begin{equation}
	\label{eq:A3}
	g_i = \frac{1}{c} \left( (\boldsymbol{B}^T\boldsymbol{\alpha})_i + \beta_i - d_i \right), \ \ \ i \in \mathbb{N}_{KN}.
	\end{equation}
	
	\noindent
	Let $\gamma_i \triangleq (\boldsymbol{B}^T\boldsymbol{\alpha})_i - d_i$. Combining \eref{eq:A3} with the complementary slackness condition $\beta_i g_i=0$, one obtains that, if $\gamma_i \leq 0$, then $\beta_i = - \gamma_i$ and $g_i = 0$, while, when $\gamma_i > 0$, then $\beta_i = 0$ and, evidently, $g_i = \frac{1}{c} \gamma_i$. These two observations can be summarized as $g_i = \frac{1}{c} \left[ \gamma_i \right]_{+}$, which completes the proof. 
\end{proof}

\noindent
In order to exploit the result of \thmref{thm1} for obtaining a concrete solution to \pref{eq:eqn09}, a binary search is employed to find the (unknown) optimal values of the Lagrange multipliers $\alpha_i$, so they satisfy the equality constraint $\boldsymbol{B}\boldsymbol{g}=\boldsymbol{1}$.



In conclusion, the entire algorithm for solving \pref{eq:R2LML01} is depicted in \aref{alg:alg2} and can be recapitulated as follows: for the first block, the $\boldsymbol{g}^k$ vectors are assumed fixed and a \ac{PSD} step is employed to minimize the cost function of \eref{eq:R2LML01} \wrt\ each weight matrix $\boldsymbol{L}^k$. In the second block, all $\boldsymbol{L}^k$'s are held fixed to the values obtained from step $1$ and the solution offered by \thmref{thm1} along with binary search solutions for the $\alpha_i$'s are used to compute the optimal $\boldsymbol{g}_k$'s by iteratively solving \pref{eq:eqn09} via a \ac{MM} scheme. These two main blocks are repeated until convergence. 

\subsection{The Three-Block Algorithm Variant for \ac{T-R2LML}}
\label{sec:ThreeStepAlgorithm}

The first two \ac{BCD} steps of \ac{T-R2LML} are identical to the ones of \ac{E-R2LML}. However, since \ac{T-R2LML} embodies a trasductive learning approach, a third \ac{BCD} step is required, in order to predict the similarities between all samples, including the ones used for testing. In specific, for the third block optimization, \pref{eq:eqn04} is minimized over $\boldsymbol{S}$ for fixed $\boldsymbol{L}^k$'s and $\boldsymbol{g}^k$'s. 
By defining 

\begin{align}
\psi_{mn} & \triangleq \sum_k \left( \left \| \boldsymbol{L}^k \Delta \boldsymbol{x}_{mn} \right \|^2_2 g_n^k g_m^k - C[1-\left \| \boldsymbol{L}^k \Delta \boldsymbol{x}_{mn}\right \|^2_2]_{+} \right),
\end{align}

\noindent
\pref{eq:eqn04} becomes:

\begin{align}
\label{eq:R2LMTL02}
\underset{\boldsymbol{S}}{min} \ & \ \mbox{trace}\left\lbrace \boldsymbol{S} \boldsymbol{\Psi} \right\rbrace \\
s.t. \ & \ s_{mn} \in \{0,1\}, \ m,n \in  \mathbb{N}_M \nonumber \\
\ & \ s_{mm}=1, \ s_{mn} = s_{nm}, \ m,n \in  \mathbb{N}_M \nonumber \\
\ & \ \underset{n \in \mathbb{N}_{N+M}}{\sum s_{mn}}  \geq 2, \ m \in \mathbb{N}_M. \nonumber
\end{align}

\noindent
where $\boldsymbol{\Psi} \in \mathbb{R}^{M \times N}$ is the matrix with elements $\psi_{mn}$. This is a $0-1$ integer programming problem. By scanning the matrix $\boldsymbol{\Psi}$ row by row, \pref{eq:R2LMTL02} will be optimally solved using the following rules:

\begin{itemize}
	\item For rows of $\boldsymbol{\Psi}$ containing at least one negative element, set the corresponding $s_{mn}$ element(s) to $1$; the remaining elements are set to $0$. 
	\item For rows of $\boldsymbol{\Psi}$ with no negative element, the $s_{mn}$ element, which corresponds to the smallest $\psi_{mn}$, is set to $1$; the remaining elements are set to $0$.
	\item Note that $s_{nm}$ must equal $s_{mn}$, since the matrix $\boldsymbol{S}$ is symmetric.
\end{itemize}

\begin{algorithm}[tb]
	\caption{Minimization of \pref{eq:eqn04} and \pref{eq:R2LML01}}
	\label{alg:alg2}
	\begin{algorithmic}
		\STATE {\bfseries Input:} Data $\boldsymbol{X} \in R^{D \times (N+M) }$ for \pref{eq:eqn04} and $\boldsymbol{X} \in R^{D \times N }$ for \pref{eq:R2LML01},  number of metrics $K$
		\STATE {\bfseries Output:} $\boldsymbol{L}^k, \boldsymbol{g}^k, \boldsymbol{S}$ (here, $\boldsymbol{S}$ is only for \pref{eq:eqn04} ) 
		\STATE 01. Initialize $\boldsymbol{L}^k, \boldsymbol{g}^k, \boldsymbol{S}$ for all $k \in \mathbb{N}_K$
		\STATE 02. {\bfseries While} not converged {\bfseries Do}
		\STATE 03.\quad Block 1: Use a \ac{PSD} method to solve \pref{eq:eqn04} for each $\boldsymbol{L}^k$
		\STATE 04.\quad Block 2:
		\STATE 05.\quad \quad $\boldsymbol{\tilde{S}} \leftarrow \eref{eq:eqn07}$
		\STATE 06.\quad \quad $\mu\leftarrow -\lambda_{max}(\boldsymbol{\tilde{S}})$
		\STATE 07.\quad \quad $\boldsymbol{H} \leftarrow \boldsymbol{\tilde{S}}+\mu \boldsymbol{I}$
		\STATE 08.\quad \quad {\bfseries While} not converged {\bfseries Do}
		\STATE 09.\quad \quad \quad Apply binary search to obtain each $\boldsymbol{g}^k$ using \eref{eq:eqn10}
		\STATE 10.\quad \quad {\bfseries End While}
		\STATE 11.\quad Block 3 (this block only for \pref{eq:eqn04}): Optimal Algorithm for \pref{eq:R2LMTL02}
		\STATE 11. {\bfseries End While}
	\end{algorithmic}
\end{algorithm}

\noindent
For the sake of completeness, the relevant algorithm is summarized in \aref{alg:alg2}. Note that these three main blocks are repeated until a preset maximum number of steps is reached. 


\subsection{Analysis}
\label{sec:Analysis}

In this subsection, we investigate the convergence of our proposed \aref{alg:alg2}. This is a local analysis since our framework is non-convex. As mentioned in previous sections, a PSD approach is used to minimize the function $f(\boldsymbol{w})+r(\boldsymbol{w})$, where both $f \triangleq \sum_k \sum_{m,n}  s_{mn}\left \| \boldsymbol{L}^k \Delta \boldsymbol{x}_{mn}\right  \|^2_2 g_n^k g_m^k + C(1-s_{mn}) \left[ 1-\left \| \boldsymbol{L}^k \Delta \boldsymbol{x}_{mn} \right \|^2_2\right]_+$ and $r \triangleq \sum_k \left \| \boldsymbol{L}^k \right \|_*$ are non-differentiable. Denote $\partial f$ as the subgradient of $f$ and define $\left \| \partial f(\boldsymbol{w}) \right \| \triangleq \sup_{\boldsymbol{g} \in \partial f(\boldsymbol{w})} \left \| \boldsymbol{g} \right \|_2$; the corresponding quantities for $r$ are similarly defined. Like in \cite{Langford2009} and \cite{ShalevShwartz2011}, the subgradients are assumed to be bounded, \ie:

\begin{equation}
\label{eq:eqn11}
\left \| \partial f(\boldsymbol{w}) \right \|^2 \leq Af(\boldsymbol{w})+G^2, \ \ \ \left \| \partial r(\boldsymbol{w}) \right \|^2 \leq Ar(\boldsymbol{w})+G^2,
\end{equation}

\noindent
where $A$ and $G$ are positive scalars. Let $\boldsymbol{w}^*$ be the minimizer of $f(\boldsymbol{w})+r(\boldsymbol{w})$. Then we have the following theorem for the problem under consideration.

\begin{thm}
	\label{lem2}
	Suppose that a \ac{PSD} method is employed to solve $\min_{\boldsymbol{w}} \{ f(\boldsymbol{w}) + r(\boldsymbol{w}) \}$. Assume that 1) $f$ and $r$ are lower-bounded; 2) the norms of any subgradients $\partial f$ and $\partial r$ are bounded as in \eref{eq:eqn11}; 3) $\left \| \boldsymbol{w}^* \right \| \leq D$ for some $D>0$; 4) $r(\boldsymbol{0})=0$. Let $\eta_t \triangleq \frac{D}{\sqrt{8T}G}$, where $T$ is the number of iterations of the PSD algorithm. Then, for a constant $c \leq 4$, such that $(1-cA\frac{D}{\sqrt{8T}D}) > 0$, and initial estimate of the solution $\boldsymbol{w}_1=\boldsymbol{0}$, we have:

	\begin{align}
	\label{eq:eqn12}
	\underset{t \in \mathbb{N}_T}{min} & \left[ f(\boldsymbol{w}_t)+  r(\boldsymbol{w}_t) \right] \leq  \frac{1}{T}  \sum_{t=1}^T f(\boldsymbol{w}_t)+r(\boldsymbol{w}_t) \leq \nonumber \\
	& \leq \frac{2\sqrt{2}DG}{\sqrt{T}(1-\frac{cAD}{G\sqrt{8T}})}+\frac{f(\boldsymbol{w}^*)+r(\boldsymbol{w}^*)}{1-\frac{cAD}{G\sqrt{8T}}}.
	\end{align}
\end{thm}

\noindent
The detailed proof of \thmref{lem2} is given in the Appendix A. \thmref{lem2} implies that, as $T$ grows, the \ac{PSD} iterates approach $\boldsymbol{w}^*$. 

\begin{thm}
	\label{thm3}
	\aref{alg:alg2} yields a convergent, non-increasing sequence of cost function values relevant to \pref{eq:R2LML01} and \pref{eq:eqn04}. Furthermore, the set of fixed points of the iterative map embodied by \aref{alg:alg2} include the \ac{KKT} points of \pref{eq:R2LML01} and \pref{eq:eqn04}.
\end{thm}

The proof is showcased in Appendix B. \thmref{thm3} implies the convergence of the two proposed algorithms.

\section{Experiments}
\label{sec:Experiments}

\subsection{Effects of our nuclear norm-based regularization}

\begin{table*}[t]
	\setlength{\tabcolsep}{5pt}
	\caption{Classification accuracy versus $\lambda$ for highly overlapping dataset.}
	\label{table0}
	\vskip 0.15in
	\begin{center}
		\begin{small}
			\begin{sc}
				\begin{tabular}{lcccccc}
					\hline
					\noalign{\smallskip}
					$\lambda$  	& $0$ 		& $0.1$ 	& $1$ 		& $10$ 		& $100$ 	& $1000$ \\
					\noalign{\smallskip}
					\hline
					\noalign{\smallskip}
					Accuracy    & $0.544$   & $0.528$	& $0.553$	& $0.563$	& $0.563$ 	& $0.534$	\\
					\hline
				\end{tabular}
			\end{sc}
		\end{small}
	\end{center}
	\vskip -0.1in
\end{table*} 

\begin{table*}[t]
	\setlength{\tabcolsep}{5pt}
	\caption{Classification accuracy versus $\lambda$ for highly overlapping dataset with sparse features. The number of columns with all $0$ entries for each metric is also reported.}
	\label{table0.5}
	\vskip 0.15in
	\begin{center}
		\begin{small}
			\begin{sc}
				\begin{tabular}{lcccccccc}
					\hline
					\noalign{\smallskip}
					$\lambda$  								& $0$ 		& $0.1$ 	& $1$ 		& $10$ 		& $10^2$ 	& $10^3$		& $10^4$		& $10^5$ \\
					\noalign{\smallskip}
					\hline
					\noalign{\smallskip}
					Accuracy            					& $0.725$   & $0.813$	& $0.747$	& $0.713$	& $0.725$	& $0.975$	& $0.916$	& $0.488$	\\
					\noalign{\smallskip}
					\hline
					\noalign{\smallskip}
					\# of zero columns in Metric $1$      	& $0$		& $0$ 		& $0$ 		& $0$		& $0$		& $13$		& $13$		& $0$ \\
					\# of zero columns in Metric $2$    	& $0$ 		& $0$ 		& $0$	 	& $0$     	& $0$		& $13$		& $13$		& $60$	\\
					\# of zero columns in Metric $3$     	& $0$ 		& $0$ 		& $0$	 	& $0$		& $0$		& $13$		& $14$		& $60$	\\
					\hline
				\end{tabular}
			\end{sc}
		\end{small}
	\end{center}
	\vskip -0.1in
\end{table*} 

Since \ac{T-R2LML} involves $(M-1)M/2 + M N + (D^2 + M + N) K$ parameters, while \ac{E-R2LML} employs $(D^2 + N) K$, we can see that both \ac{R2LML} frameworks may benefit from regularization, when confronted with scarce, high-dimensional, noisy data. 
Two synthetic datasets were created to study the effects of nuclear norm regularization. 

The first set consisted of $30$-dimensional samples, while the second one consisted of $60$-dimensional features. In both cases, samples were drawn from a mixture of two highly overlapping Gaussian distributions, whose covariance matrices had a spectral radius of $0.3$. Moreover, in the case of the second dataset, features randomly selected with probability $0.5$ were set to $0$ to emulate sparsity.
For both datasets, $80$ samples were used for training via \ac{E-R2LML} and $320$ samples for testing. 
Also, $3$ local metrics were employed, while the remaining parameters were set as follows: the \ac{PSD} step length was set to $10^{-6}$ and the algorithm was allowed to run for $5$ epochs of $500$ iterations each. The classification accuracy using a $5$-nearest neighbor search is reported in \tref{table0} and \tref{table0.5} for various values of the regularization parameters.


These two tables reflect, as expected, that the regularization proves to be very important for \ac{E-R2LML}, and, by extension, to \ac{T-R2LML} as well, since the latter one deals with additional parameters to be learned. More specifically, it is shown that cross-validation over $\lambda$ is essential in improving classification accuracy for noisy, potentially sparse, highly overlapping data. This is especially more pronounced for the second dataset, where not employing regularization is clearly inferior to the performance attained by fine-tuning $\lambda$. Also, for the same dataset, \tref{table0.5} illustrates the sparsity-inducing properties of the nuclear norm regularizer. It is worth noting that, although not specifically shown here, the metrics' all-$0$ columns obtained for $\lambda = 10^3$ and $\lambda = 10^4$ followed exactly the sparsity pattern of the relevant features.


\subsection{Real datasets}

In order to assess the utility of the proposed models, we performed experiments on $18$ datasets, namely, \textit{Robot Navigation}, \textit{Letter Recognition}, \textit{Pendigits}, \textit{Wine Quality}, \textit{Gamma Telescope}, \textit{Ionosphere}, \textit{Breast Tissue}, \textit{Glass}, \textit{Heart}, \textit{Sonar}, \textit{WPBC}, \textit{Optdigits} and \textit{Isolet} datasets from the \textit{UCI machine learning repository}\footnote{\url{http://archive.ics.uci.edu/ml/datasets.html}}, and \textit{Image Segmentation}, \textit{Two Norm}, \textit{Ring Norm} datasets from the \textit{Delve Dataset Collection}\footnote{\url{http://www.cs.toronto.edu/~delve/data/datasets.html}}. We also considered the \textit{Columbia University Image Library (COIL20)}\footnote{\url{http://www.cs.columbia.edu/CAVE/software/softlib/coil-20.php}} and \textit{USPS}\footnote{\url{http://www.gaussianprocess.org/gpml/data/}} datasets. Major characteristics of these datasets are summarized in \tref{table1}. 
Following experimental settings similar to the ones used in \cite{Wang2012} and \cite{Zhu2014}, PCA was used on the data of \textit{COIL20}, \textit{Isolet}, \textit{Optdigits} and \textit{USPS} to reduce their number of features to $30$, as shown in \tref{table1}.

We first explored how the performance of \ac{T-R2LML}\footnote{\url{https://github.com/yinjiehuang/R2LMTL/archive/master.zip}} and \ac{E-R2LML}\footnote{\url{https://github.com/yinjiehuang/R2LML/archive/master.zip}} varies with respect to the number of local metrics. Then, we compared \ac{T-R2LML} and \ac{E-R2LML} to other state-of-the-art global and local metric learning algorithms, namely, ITML, LMNN, LMNN-MM, GLML and PLML.


\begin{table*}[t]
	\setlength{\tabcolsep}{5pt}
	\caption{Details of benchmark datasets. The columns indicate number of features (\#D), classes (\#classes), number of validation (\#validation) and test (\#test) samples.}
	\label{table1}
	\vskip 0.15in
	\begin{center}
		\begin{small}
			\begin{sc}
				\begin{tabular}{lcccccc}
					\hline
					\noalign{\smallskip}
					& \#D & \#classes & \#train & \#validation & \#test \\
					\noalign{\smallskip}
					\hline
					\noalign{\smallskip}
					A. Robot            	& $4$   & $4$	& $240$	 & $240$	& $4976$	\\
					B. Letter       		& $16$	& $26$ 	& $520$	 & $2600$	& $2600$	\\
					C. Pendigits    		& $16$ 	& $10$ 	& $400$	 & $2000$   & $2000$	\\
					D. Wine Quality     	& $12$ 	& $2$ 	& $150$	 & $150$	& $2898$	\\
					E. Telescope        	& $10$ 	& $2$ 	& $300$  & $300$    & $5400$	\\
					F. Image Segmentation   & $18$ 	& $7$ 	& $210$  & $210$    & $1890$	\\
					G. Two Norm          	& $20$ 	& $2$ 	& $250$	 & $250$    & $3900$	\\ 
					H. Ring Norm         	& $20$ 	& $2$   & $250$  & $250$    & $3900$	\\
					I. Ionosphere		 	& $34$  & $2$   & $80$	 & $50$     & $221$ 	\\
					J. Breast Tissue 		& $9$	& $6$	& $18$	 & $18$		& $70$	\\
					K. COIL20				& $30$	& $20$	& $400$	 & $400$	& $640$	\\
					L. Glass				& $9$	& $6$ 	& $18$	 & $18$		& $178$ 	\\
					M. Heart				& $13$	& $2$ 	& $40$ 	 & $40$		& $190$ 	\\
					N. Isolet				& $30$	& $26$	& $520$	 & $3640$	& $3637$ 	\\
					O. Optdigits			& $32$	& $10$	& $400$	 & $2400$	& $2820$ 	\\
					P. Sonar				& $60$	& $2$	& $40$	 & $40$		& $180$ 	\\
					Q. USPS					& $30$	& $10$	& $400$	 & $4500$	& $4398$ 	\\
					R. WPBC					& $33$	& $2$	& $20$	 & $20$		& $158$ 	\\
					\hline
				\end{tabular}
			\end{sc}
		\end{small}
	\end{center}
	\vskip -0.1in
\end{table*}

\begin{figure*}[ht]
	\vskip 0.2in
	\begin{center}
		\centerline{\includegraphics[width=\textwidth]{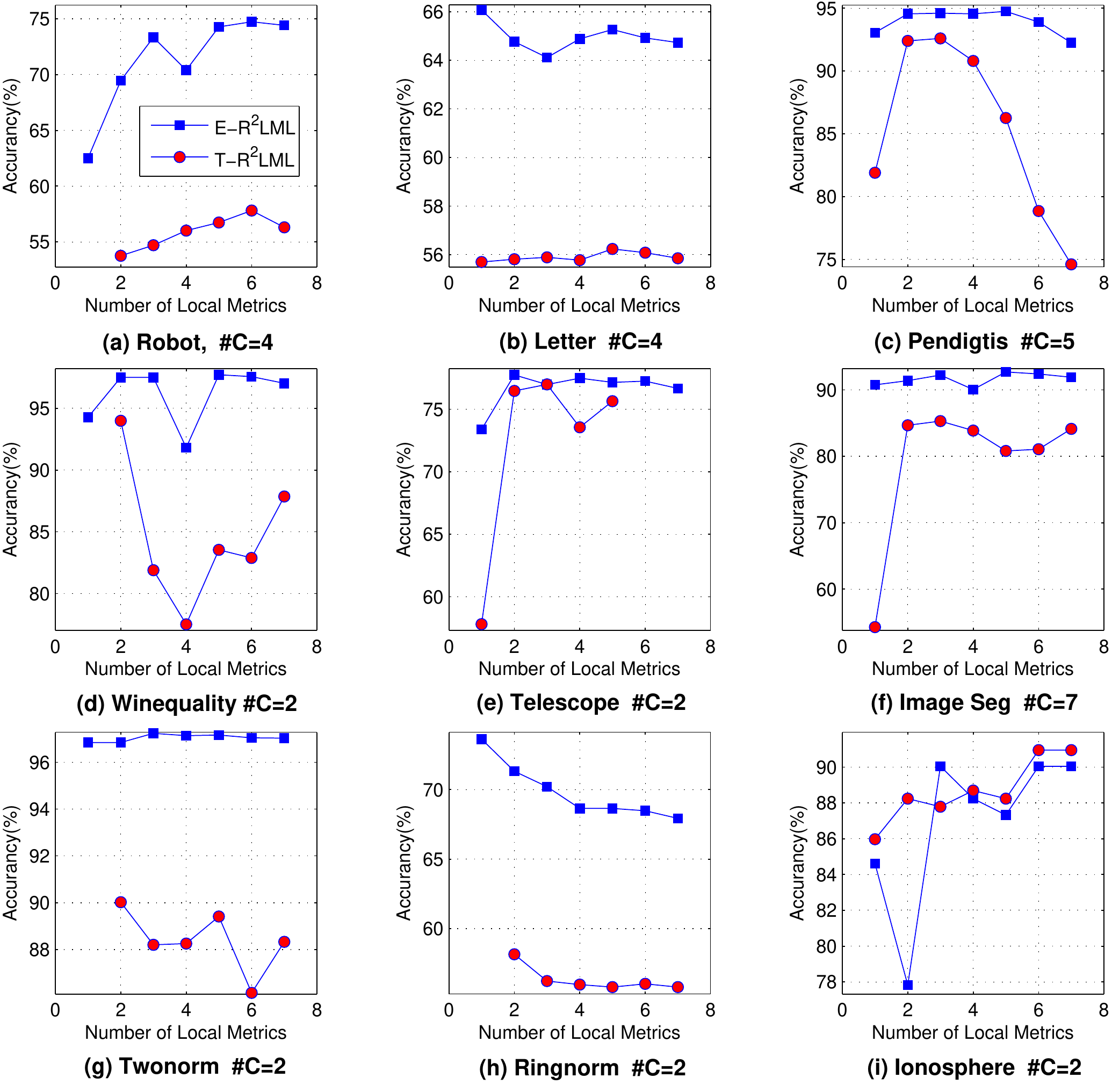}}
		\caption{\ac{T-R2LML} and \ac{E-R2LML} classification accuracy results on the first $9$ benchmark datasets for varying number $K$ of local metrics. $\#C$ indicates the number of classes of each dataset.}
		\label{figure3}
	\end{center}
	\vskip -0.2in
\end{figure*}

\subsubsection{Number of local metrics for \ac{T-R2LML} and \ac{E-R2LML}}
\label{sec:Num_R2LMTL}

One aspect that was investigated is how the performances of \ac{T-R2LML} and \ac{E-R2LML} vary with respect to the number of local metrics $K$. In \cite{Weinberger2008}, the authors set $K$ equal to the number of classes for each dataset, which might not necessarily be the optimal choice. An abundance of data may imply that more local metrics may be necessary for improved performance; this is an aspect we examined for \ac{T-R2LML} and \ac{E-R2LML}. 
For all datasets, the range of $K$ we considered was $1-7$, which, aside from \emph{COIL20}, \emph{USPS} and \emph{Isolet}, included the number of classes represented in the data. As we will argue in the sequel, the optimal $K$ does not necessarily coincide with the number of classes of the corresponding classification problem. As a matter of fact, it coincides only in roughly one quarter of the cases.  

\begin{figure*}[ht]
	\vskip 0.2in
	\begin{center}
		\centerline{\includegraphics[width=\textwidth]{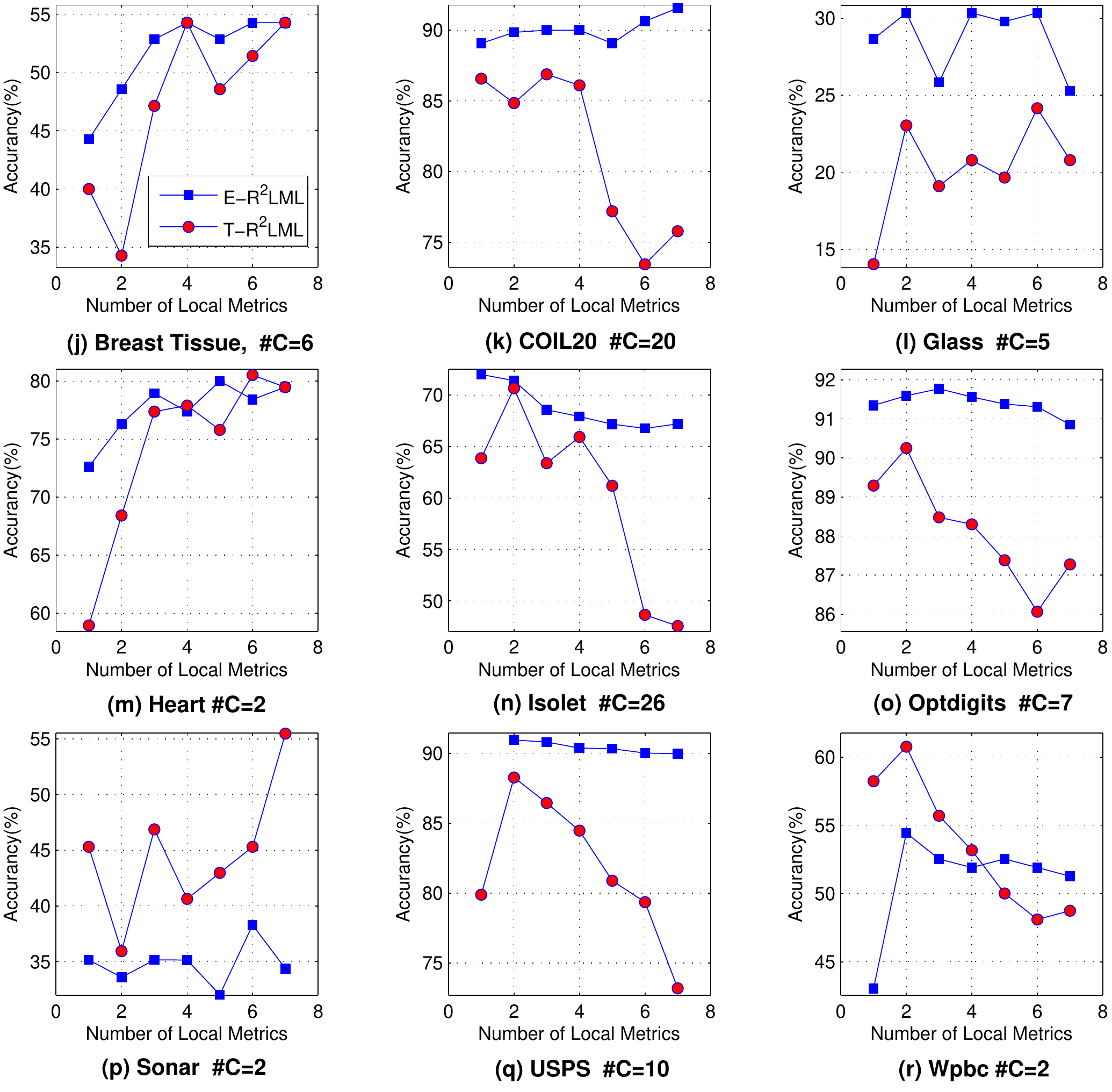}}
		\caption{\ac{T-R2LML} and \ac{E-R2LML} classification accuracy results for the remaining $9$ benchmark datasets as the number $K$ of local metrics varies. $\#C$ indicates the number of classes of each dataset.}
		\label{figure4}
	\end{center}
	\vskip -0.2in
\end{figure*}

For \ac{T-R2LML}, we set the penalty parameter $C$ to $1$ and the regularization parameter $\lambda$  to $10$. 
In the case of \ac{E-R2LML}, $\lambda$ was chosen smaller, since it employs less parameters compared to \ac{T-R2LML} and, therefore, is less prone to over-fitting. 
Note that all aforementioned parameter values were selected via cross-validation and subsequently held fixed. Moreover, we terminated our algorithm, if it reached $5$ epochs or when the difference in cost function values between two consecutive iterations was less than $10^{-4}$. In each epoch, the \ac{PSD} was ran for $500$ iterations with step length 
$10^{-5}$ for the \textit{Sonar} dataset, to $10^{-6}$ for the \textit{Ionosphere} and \textit{Glass} datasets, to $10^{-8}$ for the \textit{Ring Norm} dataset, $10^{-9}$ for the \textit{Robert}, \textit{Letter}, \textit{Two Norm} and \textit{Heart} datasets, $10^{-10}$ for the \textit{COIL20}, \textit{Isolet}, \textit{Optdigits} and \textit{USPS} datasets, to $10^{-11}$ for the \textit{Pendigits}, \textit{Image Segmentation}, \textit{Telescope}, \textit{Wine Quality} and \textit{Wpbc} datasets and $10^{-13}$ for the \textit{Breast Tissue} dataset. The \ac{MM} loop was terminated, if the number of iterations reached $3000$ or when the difference in cost function values between two consecutive iterations was less than $10^{-3}$.

For \ac{E-R2LML}, the parameters like $C$, the number of epochs and the number of iterations were set the same as \ac{T-R2LML}. 
The \ac{PSD} step length was fixed to $10^{-3}$ for the \textit{Glass} and \textit{Sonar} datasets, to $10^{-5}$ for the \textit{Robot} and \textit{Ionosphere} datasets, to $10^{-6}$ for the \textit{Letter}, \textit{Two Norm}, \textit{Ring Norm} and \textit{Optdigits} datasets, to $10^{-7}$ for the \textit{Isolet} and \textit{USPS} datasets, to $10^{-8}$ for the \textit{Wine Quality}, \textit{Image Segmentations}, \textit{COIL20} and \textit{Heart} datasets, to $10^{-9}$ for the \textit{Pendigits}, \textit{Gamma Telescope} and \textit{Wpbc} datasets and to $10^{-11}$ for the \textit{Breast Tissue} dataset. 

The relation between number of local metrics and classification accuracy for each dataset is reported in \fref{figure3} and \fref{figure4}. Several observations can be made based on these results. 
First, the results indicate that training with more data does not necessarily imply that an increased value of $K$ is needed for improved performance results. For example, in the case of \ac{T-R2LML}, for the \textit{Pendigits}, \textit{Wine Qulity}, \textit{Two Norm}, \textit{Ring Norm}, \textit{Glass}, \textit{Isolet} and \textit{Optdigits} datasets, $2$ local metrics are enough to yield the best results among other choices of $K$. When \ac{E-R2LML} is trained with the \textit{Telescope} and \textit{USPS} datasets, superior results are obtained using only $2$ metrics. Secondly, one cannot discern a deterministic relationship between the classification accuracy and the number of local metrics utilized that is suitable for all datasets. For the \textit{Ring Norm} dataset, the classification accuracy is monotonically decreasing with respect to $K$, while for the remaining datasets, the optimal $K$ varies in a non-apparent fashion with respect to their number of classes. All these observations suggest that validation over $K$ is needed to select the best performing model. Also, one discerns that, although \ac{T-R2LML} is trained with more data, \ac{E-R2LML} outperforms it on all datasets except the \textit{Telescope}, \textit{Ionosphere}, \textit{Breast Tissue}, \textit{Heart} and \textit{Wpbc} datasets. 
Finally, from the obtained results results, it becomes apparent that, using both \ac{R2LML} variants as local metric learning methods (when $K>1$) is, more often than not, advantageous compared to the case, when they are used with a single global metric (when $K=1$); this is most prominently exhibited in the case of the \textit{Heart}, \textit{Wpbc}, \textit{Ionosphere} and \textit{Telescope} datasets.

%

\subsubsection{Performance Comparisons} 
\label{sec:comp}

We compared \ac{T-R2LML} and \ac{E-R2LML} to several other metric learning algorithms, including Euclidean metric \ac{KNN}, ITML \cite{Davis2007}, LMNN \cite{Weinberger2006}, LMNN-MM \cite{Weinberger2008}, GLML \cite{Noh2010} and PLML \cite{Wang2012}. Both ITML and LMNN learn a global metric, while LMNN-MM, GLML and PLML are local metric learning algorithms. After the metrics are learned for each method, a $5$-nearest neighbor decision rule was employed to classify unlabeled samples. 

For our experiments we used LMNN, LMNN-MM\footnote{\url{http://www.cse.wustl.edu/~kilian/code/code.html}},  ITML\footnote{\url{http://www.cs.utexas.edu/~pjain/itml/}} and PLML\footnote{\url{http://cui.unige.ch/~wangjun/papers/PLML.zip}} implementations that were available online. For ITML, a good value of $\gamma$ was found via cross-validation. Also, for LMNN and LMNN-MM, the number of attracting neighbors during training was set to $1$ as suggested in the paper. Additionally, for LMNN, at most $500$ iterations were performed and $30\%$ of training data were used as a validation set. The maximum number of iterations for LMNN-MM was set to $50$ and a step size of $10^{-7}$ was used. For GLML, we chose the optimal $\gamma$ setting via cross-validation. Finally, the PLML hyper-parameter values were chosen as in \cite{Wang2012}, while $\alpha_1$ was chosen via cross-validation. For \ac{T-R2LML}, the value of the regularization parameter $\lambda$ was cross-validated over $\{10^{-1}, 1, 10^{1},..., 10^{6}, 10^{7}\}$. The other parameters values used were set as described in \sref{sec:Num_R2LMTL}. With respect to \ac{E-R2LML}, the regularization parameter $\lambda$ was chosen via a validation procedure over the set $\{10^{-2}, 10^{-1}, 1, 10^{1}, 10^{2}\}$. The remaining parameter settings of our methods were the same as the ones used in the previous experiments. Finally, for both methods, $K$, the number of metrics, is cross-validated over $\{1, 2, ..., 7\}$.   

\begin{table*}[t]
	\setlength{\tabcolsep}{5pt}
	\caption{Percent accuracy results of $8$ algorithms on $18$ benchmark datasets. For each dataset, the statistically best and comparable results for a family-wise significance level of $0.05$ are highlighted in boldface. All algorithms are ranked from best to worst; algorithms share the same rank, if their performance is statistically comparable.}
	\label{table2}
	\vskip 0.15in
	\begin{center}
		\begin{small}
			\begin{sc}
				\scalebox{0.9}
				{
					\begin{tabular}{ccccccccc}
						\hline
						\noalign{\smallskip}
						& Euclidean &ITML & LMNN & LMNN-MM & GLML & PLML & \ac{T-R2LML} & \ac{E-R2LML} \\
						\noalign{\smallskip}
						\hline
						\noalign{\smallskip}
						A & $65.31^{2nd}$ & $65.86^{2nd}$ & $66.10^{2nd}$ & $66.10^{2nd}$ & $62.28^{3rd}$ & $61.03^{3rd}$ & $58.72^{4th}$ & $\textbf{74.16}^{1st}$\\
						B & $51.42^{3dr}$ & $\textbf{63.92}^{1st}$ & $\textbf{64.73}^{1st}$ & $\textbf{64.73}^{1st}$ & $57.15^{2nd}$ & $\textbf{64.62}^{1st}$ & $57.19^{2nd}$ & $\textbf{66.96}^{1st}$\\
						C & $93.15^{2nd}$ & $92.80^{2nd}$ & $93.55^{2nd}$ & $93.70^{2nd}$ & $93.10^{2nd}$ & $\textbf{95.55}^{1st}$ & $93.10^{2nd}$ & $\textbf{94.75}^{1st}$\\
						D & $87.65^{4th}$ & $91.44^{3rd}$ & $90.13^{3rd}$ & $90.44^{3rd}$ & $91.30^{3rd}$ & $\textbf{97.48}^{1st}$  & $95.03^{2nd}$ & $\textbf{96.86}^{1st}$\\
						E & $70.02^{3rd}$ & $71.04^{2nd}$ & $70.04^{2nd}$ & $66.80^{2nd}$ & $70.00^{3rd}$ & $\textbf{77.44}^{1st}$ & $\textbf{76.89}^{1st}$ & $\textbf{77.61}^{1st}$\\
						F & $80.05^{4th}$ & $90.21^{2nd}$ & $90.74^{2nd}$ & $89.42^{2nd}$ & $87.30^{3rd}$ & $90.48^{2nd}$ & $90.16^{2nd}$ & $\textbf{92.59}^{1st}$\\
						G & $96.51^{2nd}$ & $\textbf{96.82}^{1st}$ & $96.31^{2nd}$ & $96.28^{2nd}$ & $96.49^{2nd}$ & $\textbf{97.49}^{1st}$ & $\textbf{97.51}^{1st}$ & $\textbf{97.15}^{1st}$\\
						H & $55.95^{5th}$ & $73.72^{3rd}$ & $59.28^{4th}$ & $59.28^{4th}$ & $\textbf{97.28}^{1st}$ & $75.44^{3rd}$ & $80.39^{2nd}$ & $73.51^{3rd}$\\
						I & $75.57^{3rd}$ & $\textbf{86.43}^{1st}$ & $82.35^{2nd}$ & $82.35^{2nd}$ & $71.95^{3rd}$ & $78.73^{3rd}$ & $\textbf{91.86}^{1st}$ & $\textbf{90.50}^{1st}$\\
						J & $37.14^{4th}$ & $44.29^{3rd}$ & $\textbf{55.71}^{1st}$ & $47.14^{3rd}$ & $40.00^{4th}$ & $50.00^{3rd}$ & $\textbf{54.29}^{1st}$ & $\textbf{58.57}^{1st}$ \\
						K & $85.94^{4th}$ & $89.70^{2nd}$ & $88.13^{3rd}$ & $89.53^{2nd}$ & $87.34^{3rd}$ & $82.81^{5th}$ & $88.91^{2nd}$ & $\textbf{91.56}^{1st}$ \\
						L & $10.67^{4th}$ & $26.40^{2nd}$ & $15.73^{3rd}$ & $15.73^{3rd}$ & $11.80^{4th}$ & $26.97^{2nd}$ & $\textbf{32.58}^{1st}$ & $\textbf{33.34}^{1st}$ \\
						M & $56.84^{5th}$ & $79.47^{2nd}$ & $77.89^{2nd}$ & $74.21^{3rd}$ & $62.11^{4th}$ & $78.95^{2nd}$ & $\textbf{81.05}^{1st}$ & $\textbf{81.05}^{1st}$ \\
						N & $71.19^{2nd}$ & $\textbf{74.10}^{1st}$ & $\textbf{76.08}^{1st}$ & $\textbf{75.78}^{1st}$ & $70.91^{2nd}$ & $70.25^{2nd}$ & $70.66^{2nd}$ & $72.12^{2nd}$ \\
						O & $89.79^{2nd}$ & $89.33^{2nd}$ & $\textbf{93.40}^{1st}$ & $\textbf{93.40}^{1st}$ & $89.61^2nd$ & $88.30^2nd$ & $\textbf{91.52}^{1st}$ & $\textbf{92.16}^{1st}$ \\
						P & $44.53^{4th}$ & $44.53^{4th}$ & $51.36^{2nd}$ & $51.36^{2nd}$ & $39.06^{6th}$ & $42.97^{5th}$ & $\textbf{55.47}^{1st}$ & $48.44^{3rd}$ \\
						Q & $88.09^{3rd}$ & $\textbf{90.79}^{1st}$ & $89.22^{2nd}$ & $89.43^{3rd}$ & $88.45^{3rd}$ & $\textbf{90.95}^{1st}$ & $89.90^{2nd}$ & $\textbf{90.79}^{1st}$ \\
						R & $36.08^{4th}$ & $44.94^{3rd}$ & $39.24^{3rd}$ & $32.91^{4th}$ & $53.17^{2nd}$ & $41.77^{3rd}$ & $\textbf{67.72}^{1st}$ & $55.06^{2nd}$ \\
						\hline
					\end{tabular}
				}
			\end{sc}
		\end{small}
	\end{center}
	\vskip -0.1in
\end{table*}

For pair-wise model comparisons, we employed McNemar's test. Also, since there were $8$ algorithms to be compared, we used Holm's step-down procedure as a multiple hypothesis testing method to control the Family-Wise Error Rate (FWER) \cite{Hochberg1987} of the resulting pair-wise McNemar's tests. The experimental results for a family-wise significance level of $0.05$ are reported in \tref{table2}.

Despite employing a simplistic strategy to infer the weight vector of testing data, \ac{E-R2LML} achieves the best performance for $14$ out of the $18$ datasets and outperforms its transductive version, while the other methods outperform \ac{E-R2LML} on the \textit{Ring Norm}, \textit{Isolet}, \textit{Sonar} and \textit{Wpbc} datasets. GLML's surprisingly good result for the \textit{Ring Norm} dataset is probably because GLML assumes a Gaussian mixture underlying the data generation process and the \textit{Ring Norm} dataset is a $2$-class recognition problem drawn from a mixture of two multivariate normal distributions. \ac{T-R2LML} produced best results for $9$ out of the $18$ datasets. We also notice that \ac{T-R2LML} achieves almost second best results for the remaining datasets except for \textit{Robot}. For the \textit{Ring Norm}, \textit{Sonar} and \textit{Wpbc} datasets, \ac{T-R2LML} even outperforms \ac{E-R2LML}.

Next, PLML exhibits competitive results, more specifically, best in $6$ out of the $18$ cases, but performs poorly on some datasets like \textit{COIL20} and \textit{Sonar}, even worse than \ac{KNN}. 
For \textit{Glass}, \textit{Heart}, \textit{Isolet} and \textit{Optdigits}, PLML's performance is also quite impressive; it is ranked $2^{nd}$ among the other methods. 
Regarding ITML, by using a global metric, it is ranked first for $5$ datasets. Often, ITML ranks at least $2^{nd}$ and seems to be suitable for low-dimensional datasets. 
Finally, GLML rarely performs well; according to \tref{table2}, GLML only achieves $3^{rd}$ or $4^{th}$ ranks for $9$ out of the $18$ datasets. 

Another general observation that can be made is the following: employing metric learning is almost always a good choice, since the classification accuracy of utilizing a Euclidean metric is almost always ranked last among all $8$ methods considered. Interestingly, LMNN-MM, even though being a local metric learning algorithm, does not show any significant performance advantages over LMNN (a global metric method); for some datasets, it even obtained lower classification accuracy than LMNN. It is possible that fixing the number of local metrics to the number of classes present in the dataset curtails LMNN-MM's performance. According to the obtained results, \ac{T-R2LML} and \ac{E-R2LML} yield much better performance for all datasets compared to LMNN-MM. 

\acresetall 

\section{Conclusions}
\label{sec:Conclusions}

In this paper, we proposed a new local metric learning framework, namely \ac{R2LML}. \ac{R2LML} learns $K$ Mahalanobis-based local metrics that are conically combined, so that pairs of similar points are measured as being located close to each other, in contrast to pairs of dissimilar points, for which the opposite is desired. Two variants of the framework were considered: \ac{T-R2LML} employs transductive learning to infer the conic combination of metrics to be used for assessing distances between test and training data, while \ac{E-R2LML} employs a simpler technique to accelerate the learning process. If $T$ is the number of iterations, a local analysis of the block-minimization training procedure of both variants has been shown to be convergent at a rate of $\mathcal{O}(1/\sqrt{T})$, which is typical for sub-gradient methods.


In order to show the merits of \ac{T-R2LML} and \ac{E-R2LML}, we performed a series of experiments involving $18$ benchmark classification problems. First, we studied the effect of regularization in \ac{R2LML} and showed the importance of the nuclear norm-based regularizer in providing low-rank solutions that avoid over-fitting. Second, we varied the number of local metrics $K$ and discussed its influence on classification accuracy. 
We concluded that the obtained optimal $K$ does not necessarily equal the number of classes of the dataset under consideration. Also, our results indicate that larger datasets do not necessarily require employing a large number of local metrics. Finally, in a second set of experiments, we compared \ac{T-R2LML} and \ac{E-R2LML} to several other global or local metric learning algorithms and demonstrated that our proposed framework is highly competitive. 

\section*{Acknowledgments} 

Y. Huang acknowledges partial support from a UCF Graduate College Trustees Doctoral Fellowship and \ac{NSF} grant No. 1200566. C. Li acknowledges partial support from \ac{NSF} grants No. 0806931 and No. 0963146. Furthermore, M. Georgiopoulos acknowledges partial support from \ac{NSF} grants No. 1161228 and No. 0525429, while G. C. Anagnostopoulos acknowledges partial support from \ac{NSF} grant No. 1263011. Note that any opinions, findings, and conclusions or recommendations expressed in this material are those of the authors and do not necessarily reflect the views of the \ac{NSF}. 

\acresetall 

\section*{Appendix A}
\label{app:A}


In order to solve \pref{eq:eqn04} and \pref{eq:R2LML01}, the following \ac{PSD} update scheme is used:

\begin{align}
\label{eq:eqn221}
\boldsymbol{w}_{t + \frac{1}{2}} & = \boldsymbol{w}_t - \eta \boldsymbol{g}^f_t, \\
\label{eq:eqn222}
\boldsymbol{w}_{t + 1} & =  arg \min_{w} \left\{ \frac{1}{2} \left \| \boldsymbol{w} - \boldsymbol{w}_{t + \frac{1}{2}} \right \|^2 + \eta r(\boldsymbol{w}) \right\}.
\end{align}

\noindent 
Above, $\boldsymbol{g}^f_t \in \partial f(\boldsymbol{w}_t)$ and $\eta$ is a fixed step length. \ac{PSD} first computes the unconstrained subgradient with respect to $f$. 

In the second step, we find a new $\boldsymbol{w}_t$ from the intermediate result $\boldsymbol{w}_{t + \frac{1}{2}}$. By the first order optimality condition, with the minimizer $\boldsymbol{w}$, it holds that:

\begin{equation}
\label{eq:eqn223}
\boldsymbol{0} \in \partial \Bigg \{ \frac{1}{2} \left \| \boldsymbol{w} - \boldsymbol{w}_{t + \frac{1}{2}} \right \|^2 + \eta r(\boldsymbol{w}) \Bigg \} \Bigg |_{\boldsymbol{w} = \boldsymbol{w}_{t+1}}. \nonumber
\end{equation}

\noindent 
In light of \eref{eq:eqn221}, the above property amounts to: 

\begin{equation}
\label{eq:eqn23}
\boldsymbol{0} \in \boldsymbol{w}_{t+1} - \boldsymbol{w}_t + \eta \boldsymbol{g}^f_t + \eta \partial r(\boldsymbol{w}_{t+1}).
\end{equation}

\noindent 
Since $\boldsymbol{w}_{t+1}$ is the minimizer of \eref{eq:eqn222}, there is a vector $\boldsymbol{g}^r_{t+1} \in \partial r(\boldsymbol{w}_{t+1})$ such that \eref{eq:eqn23} holds, \ie

\begin{equation}
\label{eq:eqn224}
\boldsymbol{0} = \boldsymbol{w}_{t+1} - \boldsymbol{w}_t + \eta \boldsymbol{g}^f_t + \eta \boldsymbol{g}^r_{t+1}. \nonumber 
\end{equation}

\noindent 
Finally, we have the following \ac{PSD} update rule:

\begin{equation}
\label{eq:eqn241}
\boldsymbol{w}_{t+1} = \boldsymbol{w}_t - \eta \boldsymbol{g}^f_t - \eta \boldsymbol{g}^r_{t+1}.
\end{equation}

\noindent
With the definitions of $\left \| \partial f(\boldsymbol{w}) \right \|$ and $\left \| \partial r(\boldsymbol{w}) \right \|$ in \sref{sec:Analysis}, we provide \lemmaref{lem4} as follows. Note that, unless specified otherwise, $\left \|  \cdot \right \|$ will stand for the $L_2$ norm. 


\begin{lem}
	\label{lem4}
	Assume that the subgradients of $f$ and $r$ are bounded as in \eref{eq:eqn11} for some positive scalars $A$ and $G$. Let $\eta \geq 0$ be a fixed step length and $\boldsymbol{w}^*$ be the minimizer of $f(\boldsymbol{w}) + r(\boldsymbol{w})$. Then, for a constant $c \leq 4$ we have:
	
	\begin{align}
	\label{eq:eqn25}
	2 \eta (1 - cA\eta)f(\boldsymbol{w}_t) & + 2\eta(1-cA\eta)r(\boldsymbol{w}_{t+1}) \nonumber \\
	& \leq 2\eta f(\boldsymbol{w}^*) + 2\eta r(\boldsymbol{w}^*) + \left \| \boldsymbol{w}_t - \boldsymbol{w}^* \right \|^2 \nonumber \\
	& - \left \| \boldsymbol{w}_{t+1} - \boldsymbol{w}^* \right \|^2 + 8\eta^2G^2.
	\end{align}
	
\end{lem}

\begin{proof}
	By the definition of the subgradient, $\boldsymbol{g}^r_{t+1} \in \partial r(\boldsymbol{w+1})$ and $r$'s convexity:
	
	\begin{align}
	\label{eq:eqn26}
	& \ r(\boldsymbol{w}^*) \geq r(\boldsymbol{w}_{t+1})+\left \langle \boldsymbol{g}^r_{t+1}, \boldsymbol{w}^* - \boldsymbol{w}_{t+1} \right \rangle \nonumber \\
	\Rightarrow & \ - \left \langle \boldsymbol{g}^r_{t+1}, \boldsymbol{w}_{t+1} - \boldsymbol{w}^* \right \rangle \leq r(\boldsymbol{w}^*) - r(\boldsymbol{w}_{t+1}),
	\end{align}
	
	\noindent
	Additionally, the following relations hold:
	
	\begin{align}
	\label{eq:eqn27}
	\left \langle \boldsymbol{g}^r_{t+1}, \boldsymbol{w}_{t+1} - \boldsymbol{w}_t \right \rangle & = \left \langle \boldsymbol{g}^r_{t+1}, -\eta \boldsymbol{g}^f_{t} - \eta \boldsymbol{g}^r_{t+1} \right \rangle \nonumber \\
	& \overset{\eref{eq:eqn241}}{\leq} \left \| \boldsymbol{g}^r_{t+1} \right \| \left \| \eta \boldsymbol{g}^f_{t} + \eta \boldsymbol{g}^r_{t+1} \right \| \nonumber \\
	& \leq \eta \left \| \boldsymbol{g}^r_{t+1} \right \|^2 + \eta \left \| \boldsymbol{g}^f_t \right \| \left \| \boldsymbol{g}^r_{t+1} \right \| \nonumber \\
	& \overset{\eref{eq:eqn11}}{\leq}  \eta (Ar(\boldsymbol{w}_{t+1}) + G^2) \nonumber \\
	& + \eta (A \ max \{ f(\boldsymbol{w}_t), r(\boldsymbol{w}_{t+1})\} + G^2),
	\end{align}
	
	\noindent
	where the second step is due to Cauchy-Schwarz inequality.
	
	
	\noindent
	Now we relate $\left \| \boldsymbol{w}_{t+1} - \boldsymbol{w}^* \right \|$ to $\left \| \boldsymbol{w}_t - \boldsymbol{w}^*\right \|$ as follows:
	
	\begin{align}
	& \left \| \boldsymbol{w}_{t+1} - \boldsymbol{w}^* \right \|^2 = \left \| \boldsymbol{w}_t - \eta \boldsymbol{g}^f_t - \eta \boldsymbol{g}^r_{t+1} - \boldsymbol{w}^* \right \|^2\nonumber \\
	= & \left \| \boldsymbol{w}_t - \boldsymbol{w}^* \right \|^2 - 2(\eta \left \langle \boldsymbol{g}^f_t, \boldsymbol{w}_t - \boldsymbol{w}^* \right \rangle + \eta \left \langle \boldsymbol{g}^r_{t+1}, \boldsymbol{w}_t - \boldsymbol{w}^* \right \rangle) \nonumber \\
	& + \left \| \eta \boldsymbol{g}^f_t + \eta \boldsymbol{g}^r_{t+1}\right \|^2 \nonumber \\
	= & \left \| \boldsymbol{w}_t - \boldsymbol{w}^* \right \|^2-2 \eta \left \langle \boldsymbol{g}^f_t, \boldsymbol{w}_t - \boldsymbol{w}^* \right \rangle + \eta^2 \left \| \boldsymbol{g}^f_t + \boldsymbol{g}^r_{t+1}\right \|^2 \nonumber \\
	& - 2\eta(\left \langle \boldsymbol{g}^r_{t+1}, \boldsymbol{w}_{t+1} - \boldsymbol{w}^* \right \rangle - \left \langle \boldsymbol{g}^r_{t+1}, \boldsymbol{w}_{t+1} - \boldsymbol{w}_t \right \rangle).
	\label{eq:eqn28}
	\end{align}
	
	\noindent
	In \eref{eq:eqn28}, $\eta^2 \left \| \boldsymbol{g}^f_t + \boldsymbol{g}^r_{t+1}\right \|^2$ can be bounded as follows: 
	
	\begin{align}
	\label{eq:eqn29}
	& \eta^2 \left \| \boldsymbol{g}^f_t + \boldsymbol{g}^r_{t+1}\right \|^2 \nonumber \\
	= & \ \eta^2 \left \| \boldsymbol{g}^f_t \right \|^2 + 2\eta^2 \left \langle \boldsymbol{g}^f_t, \boldsymbol{g}^r_{t+1} \right \rangle + \eta^2 \left \| \boldsymbol{g}^r_{t+1} \right \|^2 \nonumber \\
	\leq & \ \eta^2(Af(\boldsymbol{w}_t) + G^2) + 2\eta^2 A \ max\{ f(\boldsymbol{w}_t), r(\boldsymbol{w}_{t+1})\} \nonumber \\ 
	& + \eta^2(Ar(\boldsymbol{w}_{t+1}) + G^2) \nonumber \\
	= & \ \eta^2 A f(\boldsymbol{w}_t) + 2\eta^2 A \ max\{ f(\boldsymbol{w}_t), r(\boldsymbol{w}_{t+1})\} \nonumber \\
	& + \eta^2 A r(\boldsymbol{w}_{t+1}) + 4 \eta^2 G^2.
	\end{align}
	
	\noindent
	When  \eref{eq:eqn27} and \eref{eq:eqn29} are substituted into \eref{eq:eqn28}, which obtain:
	
	\begin{align}
	\label{eq:eqn30}
	\left \| \boldsymbol{w}_{t+1} - \boldsymbol{w}^* \right \|^2 & \leq  \ \left \| \boldsymbol{w}_t - \boldsymbol{w}^* \right \|^2 - 2\eta \left \langle \boldsymbol{g}^f_t , \boldsymbol{w}_t-\boldsymbol{w}^* \right \rangle \nonumber \\
	& - 2\eta \left \langle \boldsymbol{g}^r_{t+1}, \boldsymbol{w}_{t+1}-\boldsymbol{w}^* \right \rangle +   \nonumber \\
	& + \eta^2 A f(\boldsymbol{w}_t) + 3\eta^2Ar(\boldsymbol{w}_{t+1}) + \nonumber \\
	& + 4\eta^2 A \ \max \{f(\boldsymbol{w}_t), r(\boldsymbol{w}_{t+1})\} + 8 \eta^2 G^2.
	\end{align}
	

	\noindent
	The convexities of both of $f(\boldsymbol{w})$ and $r(\boldsymbol{w})$ imply that: 
	
	\begin{align}
	\label{eq:eqn31}
	- \left \langle \boldsymbol{g}^f_t, \boldsymbol{w}_t-\boldsymbol{w}^* \right \rangle & \leq f(\boldsymbol{w}^*) - f(\boldsymbol{w}_t), \\
	\label{eq:eqn32}
	- \left \langle \boldsymbol{g}^r_{t+1}, \boldsymbol{w}_{t+1}-\boldsymbol{w}^* \right \rangle & \leq r(\boldsymbol{w}^*) - r(\boldsymbol{w}_{t+1}).
	\end{align}
	
	\noindent
	The following also holds:
	
	\begin{equation}
	\label{eq:eqn33}
	\max \{ f(\boldsymbol{w}_t), r(\boldsymbol{w}_{t+1})\} \leq f(\boldsymbol{w}_t) + r(\boldsymbol{w}_{t+1}).
	\end{equation}
	
	\noindent
	By substituting \eref{eq:eqn31}, \eref{eq:eqn32} and \eref{eq:eqn33} into \eref{eq:eqn30}, we obtain
	
	\begin{align}
	\label{eq:eqn34}
	\left \| \boldsymbol{w}_{t+1} - \boldsymbol{w}^* \right \|^2 \leq & \left \| \boldsymbol{w}_t-\boldsymbol{w}^* \right \| + 8\eta^2G^2 \nonumber \\
	& + 2\eta[f(\boldsymbol{w}^*) - (1-\frac{5}{2}\eta A)f(\boldsymbol{w}_t)] \nonumber \\
	& + 2\eta [r(\boldsymbol{w}^*) - (1 - \frac{7}{2}\eta A)r(\boldsymbol{w}_{t+1})] \nonumber \\
	\leq & \left \| \boldsymbol{w}_t-\boldsymbol{w}^* \right \| + 8\eta^2G^2 \nonumber \\
	& + 2\eta[f(\boldsymbol{w}^*) - (1-c\eta A)f(\boldsymbol{w}_t)] \nonumber \\
	& + 2\eta [r(\boldsymbol{w}^*) - (1 - c\eta A)r(\boldsymbol{w}_{t+1})].
	\end{align}
	
	\noindent
	By choosing $c \leq 4$, the second inequality holds. After some algebra, one can derive \eref{eq:eqn25} from \eref{eq:eqn34}. 
	
\end{proof}

\noindent The following is the detailed proof of \thmref{lem2}:

\begin{proof}
	By \lemmaref{lem4}, we have:
	
	\begin{align}
	\label{eq:eqn35}
	2 \eta [(1 - cA \eta ) f(\boldsymbol{w}_T) - f(\boldsymbol{w}^*) ] + & 2\eta[(1-cA\eta)r(\boldsymbol{w}_{t+1}) - r(\boldsymbol{w}^*)] \nonumber \\
	\leq & \left \| \boldsymbol{w}_t - \boldsymbol{w}^* \right \|^2 - \left \| \boldsymbol{w}_{t+1} - \boldsymbol{w}^* \right \|^2 \nonumber \\
	& + 8 \eta^2 G^2.
	\end{align}
	
	\noindent
	Summing \eref{eq:eqn35} over $t = 1, \ldots, T$ we get
	
	\begin{align}
	\label{eq:eqn36}
	& \sum^T_{t = 1} 2 \eta [(1 - cA \eta ) f(\boldsymbol{w}_T) - f(\boldsymbol{w}^*) ] \nonumber \\
	& + 2\eta[(1-cA\eta)r(\boldsymbol{w}_{t+1}) - r(\boldsymbol{w}^*)] \nonumber \\
	\leq & \left \| \boldsymbol{w}_1 - \boldsymbol{w}^* \right \|^2 - \left \| \boldsymbol{w}_{T+1} - \boldsymbol{w}^* \right \|^2 + 8 T \eta^2 G^2 \nonumber \\
	\leq & \left \| \boldsymbol{w}_1 - \boldsymbol{w}^* \right \|^2 + 8 T \eta^2 G^2 \nonumber \\
	\leq & D^2 + 8 T \eta^2 G^2.
	\end{align}
	
	\noindent
	The last inequality holds because $\left \| \boldsymbol{w}^* \right \|^2 \leq D$ and $\boldsymbol{w}_1 = 0$ as described in \thmref{lem2}. For part of \eref{eq:eqn36}, it holds:
	
	
	\begin{align}
	\label{eq:eqn37}
	& \sum_{t = 1}^{T} \eta [(1 - cA\eta)r(\boldsymbol{w}_{t+1}) - r(\boldsymbol{w}^*)] \nonumber \\
	= & \sum^{T}_{t = 1} \eta[(1 - cA\eta)r(\boldsymbol{w}_t) - \boldsymbol{w}^*] - \eta [(1 - cA\eta)r(\boldsymbol{w}_1) - r(\boldsymbol{w}^*)] \nonumber \\
	& + \eta [(1 - cA\eta) r(\boldsymbol{w}_{T+1}) - r(\boldsymbol{w}^*)] \nonumber \\
	= & \sum^{T}_{t = 1} \eta[(1 - cA\eta)r(\boldsymbol{w}_t) - \boldsymbol{w}^*] + \eta (1 - cA\eta) r(\boldsymbol{w}_{T+1}) \nonumber \\
	\geq & \sum^{T}_{t = 1} \eta[(1 - cA\eta)r(\boldsymbol{w}_t) - \boldsymbol{w}^*].
	\end{align}
	
	\noindent
	The second equality holds due to the assumptions that $\boldsymbol{w}_1 = \boldsymbol{0}$ and $r(\boldsymbol{0}) = 0$. Besides, given the step length $\eta$, this term $\eta (1 - cA\eta) r(\boldsymbol{w}_{T+1})$ is larger than $0$, which establishes the last inequality. Now, when substituting \eref{eq:eqn37} back into \eref{eq:eqn36}, we get
	
	\begin{align}
	\label{eq:eqn38}
	& \sum^T_{t=1} 2\eta[(1-cA\eta)r(\boldsymbol{w}_t)-\boldsymbol{w}^*] \nonumber \\
	& + 2\eta[(1-cA\eta)r(\boldsymbol{w}_t) - r(\boldsymbol{w}^*)] \leq D^2 + 8 T \eta^2 G^2 \nonumber \\
	\Rightarrow & \ 2\eta(1-cA\eta) \sum^T_{t = 1} [f(\boldsymbol{w}_t) + r(\boldsymbol{w}_t)] \nonumber \\
	& - 2\eta T (f(\boldsymbol{w}^*) + r(\boldsymbol{w}^*)) \leq D^2 + 8 T \eta^2 G^2 \nonumber \\
	\Rightarrow & \sum^T_{t = 1} f(\boldsymbol{w}_t) + r(\boldsymbol{w}_t) \leq \frac{8 T \eta^2 G^2}{2\eta(1-cA\eta)} + \frac{T (f(\boldsymbol{w}^*) + r(\boldsymbol{w}^*))}{1-cA\eta}.
	\end{align}
	
	\noindent
	Additionally, the following holds:
	
	\begin{align}
	\label{eq:eqn39}
	\underset{t \in \{ 1 ... T \}}{min} f(\boldsymbol{w}_t) + r(\boldsymbol{w}_t) \leq \frac{1}{T} \sum^T_{t = 1} f(\boldsymbol{w}_t) + r(\boldsymbol{w}_t).
	\end{align}
	
	\noindent
	Based on \eref{eq:eqn38}, \eref{eq:eqn39} and choosing $\eta = \frac{D}{\sqrt{8T}G}$, we obtain the main result shown of \thmref{lem2}.
	
\end{proof}

\section*{Appendix B}
\label{app:B}

In this section, we provide the detailed proof of \thmref{thm3} in \sref{sec:Analysis}.

\begin{proof}
	We first prove that each of the two or three block minimizations in our algorithms decrease the objective function value under consideration. This is true for the first block minimization, according to \thmref{lem2}. For the second block, since a \ac{MM} algorithm is used, we have the following relationships:
	
	\begin{equation}
	\label{eq:eqn13}
	q(\boldsymbol{g}^*)=q(\boldsymbol{g}^* | \boldsymbol{g}^*) \leq q(\boldsymbol{g}^* | \boldsymbol{g}') \leq q(\boldsymbol{g}' | \boldsymbol{g}') = q(\boldsymbol{g}').
	\end{equation}
	
	\noindent
	This implies that the second block minimization does not increase the objective function value. The optimal algorithm for the third block also guarantees the non-increasing nature of the cost function. Since the objective function is lower-bounded, \aref{alg:alg2} converges.
	
	Next, we prove that the set of fixed points of the proposed \aref{alg:alg2} includes the \ac{KKT} points of \pref{eq:R2LML01}. Towards this purpose, suppose the algorithm has converged to a \ac{KKT} point $\left\{ \boldsymbol{L}^{k*}, \boldsymbol{g}^{k*} \right\}_{k \in \mathbb{N}_K}$; then, it suffices to show that this point is also a fixed point of the algorithm's iterative map. For notational brevity, 
	let $f_{0}(\boldsymbol{L}^k,\boldsymbol{g}^k)$, $f_{1}(\boldsymbol{g}^k)$ and $h_{1}(\boldsymbol{g}^k)$ be the cost function, inequality constraint and equality constraint of \pref{eq:R2LML01} respectively. By definition, a \ac{KKT} point will satisfy
	
	\begin{align}
	\label{eq:eqn14}
	& \boldsymbol{0} \in \partial_{\boldsymbol{L}^k} f_{0}(\boldsymbol{L}^{k*},\boldsymbol{g}^{k*}) + \bigtriangledown_{\boldsymbol{g}^k} f_{0}(\boldsymbol{L}^{k*},\boldsymbol{g}^{k*})\\
	& -(\boldsymbol{\beta}^k)^T \bigtriangledown_{\boldsymbol{g}^k} f_{1}(\boldsymbol{g}^{k*}) + \boldsymbol{\alpha}^T \bigtriangledown_{\boldsymbol{g}^k} h_{1}(\boldsymbol{g}^{k*}),  \ \ \ k \in \mathbb{N}_K.  \nonumber
	\end{align}

	\noindent
	In relation to \pref{eq:eqn08}, which the second block tries to solve, by setting the gradient of the problem's Lagrangian to $\boldsymbol{0}$, the \ac{KKT} point will satisfy the following equality:

	\begin{align}
	\label{eq:eqn15}
	2\boldsymbol{\tilde{S}}\boldsymbol{g}^*-\boldsymbol{\beta}-\boldsymbol{B}^T\boldsymbol{\alpha} = \boldsymbol{0}.
	\end{align}
	
	\noindent
	\pref{eq:eqn09} can be solved based on \eref{eq:A3} of \thmref{thm1}; in specific, we obtain that 
	
	\begin{equation}
	\label{eq:extra}
	\boldsymbol{g} = - \frac{1}{2\boldsymbol{\mu}}(\boldsymbol{B}^T\boldsymbol{\alpha}+\boldsymbol{\beta}-2\boldsymbol{H}\boldsymbol{g}^*).
	\end{equation}
	
	\noindent
	Substituting \eref{eq:eqn15} and $\boldsymbol{H}=\boldsymbol{\tilde{S}}+\mu \boldsymbol{I}$ into \eref{eq:extra}, one immediately obtains that
	
	\begin{align}
	\label{eq:eqn16}
	\boldsymbol{g} & = - \frac{1}{2\boldsymbol{\mu}}(\boldsymbol{B}^T\boldsymbol{\alpha}+\boldsymbol{\beta}-2\boldsymbol{H}\boldsymbol{g}^*) \nonumber \\
	& = - \frac{1}{2\boldsymbol{\mu}}(2\boldsymbol{\tilde{S}}\boldsymbol{g}^*-2\boldsymbol{\tilde{S}}\boldsymbol{g}^*-2\boldsymbol{\mu}\boldsymbol{g}^*) = \boldsymbol{g}^*.
	\end{align} 
	
	\noindent
	In other words, step $2$ of \aref{alg:alg2} will not update the solution. Now, if we substitute \eref{eq:eqn15} back into \eref{eq:eqn14}, we obtain $\boldsymbol{0} \in \partial_{\boldsymbol{L}^k}f_{0}(\boldsymbol{L}^{k*},\boldsymbol{g}^{k*})$ for all $k$, which is the optimality condition for the subgradient method; the PSD step (the first block minimization of \aref{alg:alg2}) will also not update the solution. Thus, a \ac{KKT} point of \pref{eq:R2LML01} is a fixed point of our algorithm. 
	
	Finally, we prove that the set of fixed points of the proposed \aref{alg:alg2} includes the \ac{KKT} points of \pref{eq:eqn04}. We assume the algorithm has converged to a \ac{KKT} point $\left\{ \boldsymbol{L}^{k*}, \boldsymbol{g}^{k*} \right\}_{k \in \mathbb{N}_K}$ and $S^*$ is the true similarity matrix. Similar to the previous proof, we start from the second block. Following the same procedure, we find the second block will not update the solution of vector $\boldsymbol{g}^*$. Now, during the third block minimization, the $\psi_{mn}$ quantities remain unchanged, since $\boldsymbol{g}^*$ does not change. The minimization procedure we proposed for the third block will leave the similarity matrix unchanged, since the coefficient matrix with elements $\psi_{mn}$ is fixed. Now, if \eref{eq:eqn15} is substituted back into \eref{eq:eqn14}, we obtain the optimality condition for the first block minimization. Thus, the first block will also not update the solution. Therefore, a \ac{KKT} point of \pref{eq:eqn04} is a fixed point of \aref{alg:alg2}.
	
\end{proof}

\bibliographystyle{spbasic}
\bibliography{TNNLS2018paper}

\end{document}